\documentclass{article}

\usepackage{microtype}
\usepackage{graphicx}
\usepackage{subfig}
\usepackage{booktabs} 

\usepackage{hyperref}

\usepackage[accepted]{icml2019}

\icmltitlerunning{Optimal Kronecker-Sum Approximation of Real Time Recurrent Learning}

\usepackage{xcolor}

\newcommand{\MG}[1]{\comment{\textcolor{blue}{MG: #1}}}
\newcommand{\FB}[1]{\comment{\textcolor{red}{FB: #1}}}

\usepackage{amsmath}
\usepackage{amssymb}
\usepackage{amsthm}
\usepackage{dsfont}
\newcommand{\E}{\mathrm{E}}
\newcommand{\R}{\mathbb{R}}
\newcommand{\Var}{\mathrm{Var}}
\newcommand{\Id}{\mathrm{Id}}
\newcommand{\Tr}{\mathrm{Tr}}
\newcommand{\diag}{\mathrm{diag}}
\usepackage{caption}

\newtheorem{thm}{Theorem}
\newtheorem{lemma}{Lemma}
\newtheorem{defi}{Definition}
\newtheorem{obs}{Observation}

\usepackage{arydshln}

\begin{document}

\twocolumn[
\icmltitle{Optimal Kronecker-Sum Approximation of Real Time Recurrent Learning}



\icmlsetsymbol{equal}{*}

\begin{icmlauthorlist}
\icmlauthor{Frederik Benzing}{equal,to}
\icmlauthor{Marcelo Matheus Gauy}{equal,to}
\icmlauthor{Asier Mujika}{to}
\icmlauthor{Anders Martinsson}{to}
\icmlauthor{Angelika Steger}{to}
\end{icmlauthorlist}

\icmlaffiliation{to}{Department of Computer Science, ETH Zurich, Zurich, Switzerland}

\icmlcorrespondingauthor{FB}{benzingf@inf.ethz.ch}
\icmlcorrespondingauthor{MMG}{marcelo.matheus@inf.ethz.ch}

\icmlkeywords{Machine Learning, ICML}

\vskip 0.3in
]



\printAffiliationsAndNotice{\icmlEqualContribution} 

\begin{abstract}
One of the central goals of Recurrent Neural Networks (RNNs) is to learn long-term dependencies in sequential data. 
Nevertheless, the most popular training method, Truncated Backpropagation through Time (TBPTT), categorically forbids learning dependencies beyond the truncation horizon.
In contrast, the online training algorithm Real Time Recurrent Learning (RTRL) provides untruncated gradients, with the disadvantage of impractically large computational costs. 
Recently published approaches reduce these costs by providing noisy approximations of RTRL. 
We present a new approximation algorithm of RTRL, Optimal Kronecker-Sum Approximation (OK).
We prove that OK is optimal for a class of approximations of RTRL, which includes all approaches published so far. 
Additionally, we show that OK has empirically negligible noise: Unlike previous algorithms it matches TBPTT in a real world task (character-level Penn TreeBank) and can exploit online parameter updates to outperform TBPTT in a synthetic 
string memorization task. 
Code available at \href{https://github.com/marcelomatheusgauy/optimal_kronecker_approximation}{GitHub}.
\end{abstract}

\section{Introduction}
\label{sec:intro}
	Learning to predict sequential and temporal data is one of the core problems of Machine Learning arising for example in language modeling, speech generation and Reinforcement Learning. One of the main aims when modeling sequential data is to capture long-term dependencies. Most of the significant advances towards this goal have been achieved through Recurrent Neural Nets (RNNs). More specifically, different architectures (e.g. the Long Short-Term Memory (LSTM)~\cite{hochreiter1997long} and the Recurrent Highway Network (RHN)~\cite{zilly2017recurrent}) were developed to facilitate learning long-term dependencies and achieved notable successes.
	However, few improvements have been made regarding the training methods of RNNs. Since Williams and Peng~\yrcite{Williams90anefficient} developed Truncated Backpropagation through Time (TBPTT), it continues to be the most popular training method in many areas \cite{mnih2016asynchronous, mehri2016samplernn, merity2018analysis} - despite the fact that it does not seem to align well with the goal of learning arbitrary long-term dependencies.
	This is because TBPTT `unrolls' the RNN only for a fixed number of time steps $T$ (the truncation horizon) and backpropagates the gradient for these steps only. This almost categorically forbids learning dependencies beyond the truncation horizon. Unfortunately, extending the truncation horizon makes
	TBPTT increasingly memory consuming, since long input sequences need to be stored, and considerably slows down learning, since parameters are updated less frequently, a phenomenon known as `update lock'~\cite{jaderberg2016decoupled}.

	An alternative avoiding these issues of TBPTT is Real Time Recurrent Learning (RTRL)~\cite{williams1989learning}. The advantages of RTRL are that it provides untruncated gradients, which in principle allow the network to learn arbitrarily long-term dependencies,  and that it is fully online, so that parameters are updated frequently allowing faster learning. However, its runtime and memory requirements scale poorly with the network size and make RTRL infeasible for practical applications. As a remedy to this problem, Tallec and Ollivier~\yrcite{tallec2017unbiased} proposed replacing the full gradient of RTRL by an unbiased, less computationally costly but noisy approximation (Unbiased Online Recurrent Optimisation, UORO). Recently, Mujika et al.~\yrcite{mujika2018approximating} reduced the noise of this approach and demonstrated empirically that their improvement (Kronecker factored RTRL, KF-RTRL) allows learning complex real world data sets (character-level Penn TreeBank~\cite{marcus1993building}).
	Nevertheless, the noise introduced by KF-RTRL remains a problem, leading to slower learning and worse performance when compared to TBPTT. 
	
	To address this problem, we  propose a new approximation of RTRL, Optimal Kronecker-Sum Approximation (OK).  Extending ideas of KF-RTRL, it approximates the gradient by a sum of Kronecker-factors. It then introduces a novel procedure to perform the online updates of this approximation. We prove that this procedure has minimum achievable variance for a certain class of approximations, which includes UORO and KF-RTRL.
	Thus, OK does not only improve all previous approaches but explores the theoretical limits of the current approximation class. Empirically, we demonstrate that OK reduces the noise to a negligible level: In contrast to previous algorithms, OK matches the performance of TBPTT on a standard RNN benchmark (character-level Penn TreeBank) and also outperforms it on a synthetic string memorization task exploiting online parameter updates.  Similarly to KF-RTRL, OK is applicable to a subclass of RNNs. Besides standard RNNs this includes LSTMs and RHNs thereby covering some of the most widely used architectures.	
	
	Our theoretical findings include a construction (and proof) of a minimum-variance unbiased low-rank approximator of an arbitrary matrix, which might be applicable in other contexts of Machine Learning relying on unbiased gradients. 
	
	As a more exploratory contribution, we develop another algorithm, Kronecker-Triple-Product (KTP). Its main novelty is to match the runtime and memory requirements of TBPTT, even when measured per batch-element. Our experiments show that KTP is more noisy than OK, but that it can learn moderate time dependencies. In addition, we design an experiment to suggest directions for further improvements. 
	
	\section{Related Work}
	\FB{Feel free to add and change as you like}The most prominent training method for RNNs is Truncated Backpropagation through Time (TBPTT)~\cite{Williams90anefficient}, often yielding good results in practice. It calculates truncated gradients forbidding the network to learn long-term dependencies beyond the truncation horizon. The untruncated version of this algorithm, Backpropagation Trough Time (BPTT)~\cite{Rumelhart1986learning}, stores all past inputs and unrolls the network from the first time step, often making its computational cost unmanageable.
	
	We now review some alternatives to TBPTT. Besides Real Time Recurrent Learning and its approximations, which will be described in detail in the next section, this
	includes Anticipated Reweighted Backpropagation~\cite{tallec2017unbiasing} which samples different truncation horizons and weights the obtained gradients to calculate an overall unbiased gradient. Sparse Attentative Backtracking~\cite{ke2018sparse} uses an attention mechanism~\cite{vaswani2017attention} and  propagates the gradient along paths with high attention to extend the time span of learnable dependencies. 
	
	Other ideas avoid unrolling the network. For example, Decoupled Neural Interfaces~\cite{jaderberg2016decoupled} use neural nets to learn to predict future gradients, while Ororbia et al.~\yrcite{ororbia2018continual} propose a predictive coding based approach.
	
	For RNNs where the hidden state converges, it is also possible to avoid BPTT as shown for example by Recurrent Backpropagation~\cite{liao2018reviving} and the closely related Equilibrium Propagation~\cite{scellier2017equilibrium}.  
	
	Another approach fixes the recurrent weights and only trains the output weights. This is known as Reservoir computing~\cite{lukovsevivcius2009reservoir} and was applied for example in ~\cite{jaeger2001echo, maass2002real}. 
	
	\section{RTRL and its Approximations}\label{sec:rtrlapp}
	In this section, we derive Real Time Recurrent Learning (RTRL)~\cite{williams1989learning} and provide a common framework describing approximation algorithms for RTRL. We then embed previous algorithms and our contribution into this framework. 
	The class of approximators for which OK is optimal, as well as a precise optimality statement are given in section \ref{sec:OK}, Definition \ref{def:KS} and Theorem \ref{thm:opt}.
	The section concludes with a theoretical comparison of the different approximation algorithms including concrete examples illustrating the advantages of OK.
	
	Since the two main goals of approximating RTRL are (a) providing unbiased estimates of the gradient with (b) as little noise as possible,  we make these notions precise. For a matrix $A$ and a random variable $A'$, we say that $A'$ is an unbiased approximator of $A$ if $\E [A'] = A$ and define the noise/variance of $A'$ to be $\Var[A']=\E \left[\lVert A'-A\rVert^2\right]$, where we use the Frobenius norm for matrices.
	
	\subsection{RTRL and a General Approximation Framework}
	We start by formally defining RNNs before deriving RTRL. A RNN maintains a hidden state $h_t$ across several time steps. The next hidden state $h_{t+1}$ is computed as a differentiable function $f$ of $h_t$, the input $x_{t+1}$ and a set of learnable parameters $\theta$, $h_{t+1}=f(x_{t+1},h_t,\theta)$. Predictions for the desired output (for example predicting the next character of a given text) are a function of $h_t$ and $\theta$. We aim to minimize some loss function $L_t$ of our predictions and therefore compute $\frac{dL_t}{d\theta}$ to perform gradient descent on $\theta$. 
	
	To derive RTRL, we use the chain rule to rewrite $\frac{dL_t}{d\theta} = \frac{d L_t}{d h_t}\frac{dh_t}{d\theta}$. Next we observe
	\[
	\frac{dh_t}{d\theta} = \frac{\partial h_t}{\partial x_t}\frac{dx_t}{d\theta}+
	\frac{\partial h_t}{\partial h_{t-1}}\frac{dh_{t-1}}{d\theta}+
	\frac{\partial h_t}{\partial \theta}\frac{d\theta}{d\theta}.
	\]
	Writing $G_t:= \frac{dh_t}{d\theta}, H_t = \frac{\partial h_t}{\partial h_{t-1}}$ and $F_t = \frac{\partial h_t}{\partial \theta}$, this simplifies to (note $dx_t/d\theta =0$ as $x_t$ does not depend on $\theta$)
	\begin{eqnarray}\label{eq:RTRLREC}
	G_t = H_t G_{t-1} + F_t.
	\end{eqnarray}
	RTRL simply calculates and stores $G_t$ at each time step using the recurrence (\ref{eq:RTRLREC}) and uses it to calculate $\frac{dL}{d\theta}= \frac{d L_t}{d h_t}\frac{dh_t}{d\theta}$. This shows that RTRL is fully online and can perform frequent parameter updates.
	
	However, we already see why RTRL is impractical for applications. For a standard RNN with $n$ hidden units and $n^2$ parameters, $G_t$ has dimensions $n\times n^2$ and we need to evaluate the matrix multiplication $H_t G_{t-1}$, so that RTRL requires memory $n^3$ and runtime $n^4$ per batch element. This contrasts with TBPTT, which needs memory $Tn$ and runtime $Tn^2$, where $T$ is the truncation horizon~\cite{Williams90anefficient}.  To make RTRL competitive, we therefore need to find computationally efficient approximations.
	
	The core of RTRL is the recurrence equation (\ref{eq:RTRLREC}) and it should be the focus of any approximation. Previous approaches to approximate RTRL can be summarised as follows (see also Algorithm \ref{alg:RTRLapp}): Firstly, decide on a format in which the approximator $G'_t$ of the gradient $G_t$ is stored. This format should require less memory than storing all $n\times n^2$ numbers explicitly. Secondly, assuming $G'_{t-1}$ is given in the desired format, bring $H_{t} G'_{t-1}$ and $F_t$ in this format. Thirdly, `mix' the two terms $H_t G'_{t-1}$ and $F_t$, to bring their sum in the desired format. 
	
	\begin{algorithm}
		\caption{One step of unbiasedly approximating RTRL.\\
			\textit{This algorithm describes a framework for approximating RTRL. It assumes that the approximation is stored in a given format, called $\circledast$, and that routines $Ap(\cdot),Mix(\cdot,\cdot)$ for bringing (sums of) matrices into this format are known.}}
		\label{alg:RTRLapp}
		\begin{algorithmic}
			\STATE {{\bfseries Input:} input $x_t$, hidden state $h_{t-1}$, parameters $\theta$, unbiased approximator $G'_{t-1}$ of $G_{t-1}$ stored in prescribed format $\circledast$}.
			\STATE{{\bfseries Output:} hidden state $h_t$, unbiased approximator $G'_t$ of $G_t$ in format $\circledast$}.
			\STATE{\bfseries /* Preliminary calculations */}
			\STATE{$h_t \gets$ new hidden state, based on $h_{t-1},\theta,x_t$} 
			\STATE{$H\gets \frac{\partial h_t}{\partial h_{t-1}},\; F\gets \frac{\partial h_t}{\partial \theta}$}
			\STATE{\bfseries /* Bring addends $H G'_{t-1}, F_t$ in desired format $\circledast$*/}
			\STATE{$A_1\gets Ap(H G'_{t-1})$, where $Ap(x)$ is an unbiased approximator of $x$ in format $\circledast$}
			\STATE{$A_2\gets Ap\left(F\right)$}
			\STATE{\bfseries /* Mix two addends $A_1, A_2$ */}
			\STATE{$G'_t\gets Mix(A_1,A_2)$, where $Mix(x,y)$ is an unbiased approximator of $x+y$ in format $\circledast$}			
		\end{algorithmic}
	\end{algorithm}
	
	
	In order to obtain convergence guarantees for gradient descent with noisy gradients, it is crucial that the approximator ${G'_t}$ be unbiased and we therefore make all approximations unbiased. In the appendix (A.2.3), we empirically evaluate the difference between biased and unbiased approximators. 
	
	Another important consideration for the convergence of RTRL and its approximations is the boundedness of gradients and noise, which could accumulate indefinitely over time. Under reasonable assumptions, which are standard to avoid the exploding gradient issue in RNNs \cite{pascanu2013difficulty}, it is shown in Theorem 1 of \cite{mujika2018approximating}  that the approximations of RTRL are stable over time. This result applies to all the approximations presented below.

	\subsection{Unbiased Online Recurrent Optimization (UORO)}
	We now present UORO~\cite{tallec2017unbiased}, the first approximation algorithm of RTRL, see also \cite{cooijmans2019variance} for an analysis of its variance and related ideas. UORO follows the framework described above. It stores the approximation $G'_{t-1}$ as the outer-product (or Kronecker-product) of two vectors\footnote{For concreteness, we consider a standard RNN with $n$ hidden units and $n^2$ parameters, so that $G_t$ has dimension $n\times n^2$.} $u_{t-1},v_{t-1}$ of dimensions $n$ and $n^2$, i.e. $G'_{t-1} = u_{t-1} \otimes v_{t-1}$. Next, it rewrites $H_{t}G'_{t-1}$ observing $H_{t} (u_{t-1}\otimes v_{t-1}) = (H_{t}u_{t-1})\otimes v_{t-1} =\overline{u}_{t-1}\otimes v_{t-1}$. We omit the details of approximating $F_{t}$ by a product $r_t\otimes s_t$ and simply note that this process creates noise. We now explain how the two terms $\overline{u}_{t-1}\otimes v_{t-1}$ and $ r_{t}\otimes s_{t}$ are `mixed' in an unbiased way. This can be achieved by the so called `sign-trick'. This means choosing a uniformly random sign $c\in\{\pm 1\}$ and writing\footnote{We remark that UORO additionally introduces a variance reduction technique which rescales the factors of each outer-product to have the same norm.} $G'_{t} = (\overline{u}_{t-1} + c\cdot r_t)\otimes(v_{t-1} + c\cdot s_t)$. A simple calculation shows 
	\[
	E\bigl[(\overline{u}_{t-1} + c\cdot r_t)\otimes(v_{t-1} + c\cdot s_t)\bigr] = \overline{u}_{t-1}\otimes v_{t-1} + r_t\otimes s_t,
	\]
	so that $G'_{t+1}$ is an unbiased approximator of $H_t G'_{t-1} + F_t$. Induction on $t$ and linearity of expectation now show that $G'_t$ is an unbiased approximator of $G_t$. It is easily checked that UORO needs runtime and memory of order $n^2$.

\subsection{Kronecker Factored RTRL (KF-RTRL)}\label{sec:RTRL}
	The algorithm KF-RTRL~\cite{mujika2018approximating} is similar in spirit to UORO. The main difference is that it approximates $G'_t$ as the Kronecker-product of a vector $u_t\in\R^{1\times n}$ and a matrix $A_t\in\R^{n\times n}$, i.e. $G'_t = u_t \otimes A_t$. While this looks equivalent to UORO at first glance, Mujika et al.~observed that for many RNN architectures, including standard RNNs, RHNs and LSTMs, it is possible to  factor $F_t$ as a Kronecker-product $F_t = h_t\otimes D_t$, without adding any noise. Here, $D_t$ is a diagonal matrix. Similarly to UORO, we can exploit properties of the Kronecker-product to rewrite $H_t G'_{t-1} = u_{t-1} \otimes (H_t A_{t-1}) $ in the desired format and use a sign trick to mix the two addends $H_t G'_{t-1}$ and $F_t$ to obtain $G'_t$. 
	
	Note that KF-RTRL has memory requirements of roughly $n^2$ for storing the matrix $A_t$ and runtime of order $n^3$ due to the matrix-matrix multiplication $H_t A_{t-1}$. No additional memory is required to obtain $\frac{dL_t}{d\theta}$ as we can write $\frac{dL_t}{d\theta} = \frac{d L_t}{d h_t}\cdot G'_t = \frac{d L_t}{d h_t} \cdot \left(u_t\otimes A_t\right) =  u_t\otimes \left(\frac{d L_t}{d h_t}A_t\right)$.
	\FB{should we include the following? On the one hand, theory-minded people might be interested, on the other hand i don't think it provides much insight}

\subsection{Optimal Kronecker-Sum Approximation (OK)}
\label{sec:OK}	
	We now describe our algorithm OK. The calculations carried out by OK are reasonably simple as can be seen in the pseudo-code below. However, their correctness is not immediate and relies on the proof given in the appendix. We give some intuition and concrete examples illustrating the improvements of OK in Section \ref{subsec:Comparison}, which can be read independently of the detailed implementation.
	
	Let us start by briefly reconsidering the previous two algorithms. In our framework, there are two noise sources for approximating RTRL. Firstly, we rewrite addends $F_t$ and $H_t G_{t-1}$ in the desired format and secondly, we mix them. The first noise source is eliminated by KF-RTRL since it factors $F_t = h_t\otimes D_t$ and $H_t G'_{t-1} = u_{t-1}\otimes (H_t A_{t-1})$ noiselessly. The second noise source, stemming from mixing the terms, is the focus of our algorithm, and we shall prove below that OK not only improves previous algorithms in this step but has minimum achievable variance. 
	
	We also note that UORO and KF-RTRL both approximate $G_t$ by a `$1$-Kronecker-Sum' $G'_t$ as defined below.
	\begin{defi}[Kronecker-Sum, format]
		\label{def:KS}
		For a matrix $G\in\R^{m\times n}$, we say that $G$ is given as a $r$-Kronecker-Sum, if we are given $u_1,\ldots,u_r\in\R^{a\times b}$ and $A_1,\ldots, A_r\in \R^{c\times d}$ with $ac = m, bd = n$ so that $G = \sum_{i=1}^r u_i \otimes A_i$. We refer to $(a,b,c,d)$ as the format of the Kronecker-Sum.
	\end{defi}
	\subsubsection{Outline of OK}
	Our new algorithm, OK, has a parameter $r$, and approximates $G_t$ by a $r$-Kronecker-Sum, where each summand is the product of a vector $u_i\in \R^{1\times n}$ and a matrix $A_i\in \R^{n\times n}$, similar to KF-RTRL. 
	Concretely, we have $G'_{t-1} = \sum_{i=1}^r u_i\otimes A_i$.
	We will refer to the algorithm as $r$-OK, or simply OK depending on the context.
	 Usually, $r$ is a small constant and much smaller than the network size $n$.  
	 
	 Analogously to KF-RTRL, we focus on situations where $F_t=h_t\otimes D_t$ can be factored as a Kronecker product. A precise condition~\cite{mujika2018approximating}[Lemma 1] for when this is possible is given in the appendix (A.0.1). When we can factor $F_t=h_t\otimes D_t$, the remaining task for OK is to unbiasedly approximate the $(r+1)$-Kronecker-Sum 
	 \begin{eqnarray}\label{eq:product}
	 G = u_1\otimes (H_t A_1) +\ldots + u_r \otimes (H_t A_r) + h\otimes D
	 \end{eqnarray}
	  by a $r$-Kronecker-Sum $G'_{t}$. Equivalently, OK finds random vectors $u'_1,\ldots,u'_r$ and  matrices $A'_1,\ldots,A'_r$ so that for $$G' = \sum_{i=1}^r u'_i \otimes A'_i$$ we have $\E[G'] = G$. We now state the main optimality property of our algorithm.
	\begin{thm}
		\label{thm:opt}
		Let $G$ be an $(r+1)$-Kronecker-Sum and let $G'$ be the random $r$-Kronecker-Sum constructed by OK. Then $G'$ unbiasedly approximates $G$. Moreover, for any random $r$-Kronecker-Sum $Y$ of the same format as $G'$ which satisfies $\E[Y] = G$,  it holds that $\Var[Y]\ge \Var[G']$.
	\end{thm}
	We defer the proof to the appendix and only describe the main ideas for constructing $G'$. The first step, carried out in Algorithm \ref{alg:OK}, is to use linear algebra to reduce the problem to the following: Given a matrix $C\in\R^{(r+1)\times(r+1)}$, find a minimum-variance, unbiased approximator $C'$ of $C$, so that the (matrix-)rank of $C'$ is always at most $r$, i.e. $C'$ can be factored as $L'R'^T$ for some $L',R'\in\R^{(r+1)\times r}$.
	The next two steps are handled by Algorithm \ref{alg:Opt}: It calculates the singular value decomposition (SVD) of $C$, so that it remains to approximate a diagonal matrix $D$, and then constructs an optimal approximator $D'$ of $D$. In the appendix, we give a duality argument to prove that $D'$ is indeed optimal. 
	
	In total, the runtime of OK is of order $rn^3$, due to performing $r$ matrix-matrix multiplications (see equation \eqref{eq:product}), and the memory requirement is of order $rn^2$. The cost of calculating the optimal approximator $G'$ is asymptotically negligible.
	
	We also state the following, more general theorem. It might be useful in other settings where unbiased gradient approximations are important. Its proof is given in the appendix.
	\begin{thm}
		Given $C\in\R^{m\times n}$ and  $r\le \min\{m,n\}$, one can (explicitly) construct an unbiased approximator $C'$ of $C$, so that $C'$ always has rank at most $r$, and so that $C'$ has minimal variance among all such unbiased, low-rank approximators. This can be achieved asymptotically in the same runtime as computing the SVD of $C$.
	\end{thm}

	\subsubsection{Details of OK}
	Here, we present pseudo-code for OK (Algorithm \ref{alg:OK}). We make use of the algorithm $\mathrm{SVD}$, a standard Linear Algebra algorithm~\cite{golub1996matrix, cline2006computation} calculating the singular value decomposition of a matrix $C$. $\mathrm{SVD}$ finds a diagonal matrix $D$ and orthogonal matrices $U,V$ so that $C=UDV^T$. We only apply $\mathrm{SVD}$ to `small' matrices $C\in\R^{(r+1)\times(r+1)}$, where it needs runtime $O\left(r^3\right)$ and memory $O\left(r^2\right)$. 
	
	
\begin{algorithm}[t]
	\caption{The OK approximation}
	\label{alg:OK}
	\begin{algorithmic}
		\STATE{{\bfseries Input:} Vectors $u_1,\ldots, u_{r+1}$ and matrices $A_1,\ldots,A_{r+1}$}
		\STATE{{\bfseries Output:} Random vectors $u'_1,\ldots, u'_r$ and matrices $A'_1\ldots A'_r$, such that $\sum_{i=1}^r u'_i\otimes A'_i$ is an unbiased, minimum-variance approximator of $\sum_{i=1}^{r+1} u_i\otimes A_i$}
		\vspace{3pt}
		\STATE{\bfseries /*Rewrite in terms of orthonormal basis (onb)*/}
		\STATE{$v_1,\ldots,v_{r+1}\gets$ onb of $\text{span}\{u_1,\ldots,u_{r+1}\}$}
		\STATE{$B_1,\ldots,B_{r+1}\gets$ onb $\text{span}\{A_1,\ldots,A_{r+1}\}$}
		\FOR{$1\leq i,j \leq r+1$}
		\STATE{$L_{i,j}\gets \langle v_i,u_j\rangle,\quad R_{i,j}\gets \langle B_i,A_j\rangle$}
		\ENDFOR
		\vspace{3pt}
		\STATE{\bfseries /*Find optimal rank $r$ approximation of matrix $C$ */ }
		\STATE{$C\gets  LR^T$}
		\STATE{$(L',R')\gets Opt(C)$} \COMMENT{see Algorithm \ref{alg:Opt} for $Opt(\cdot)$}
		\vspace{3pt}
		\STATE{\bfseries /*Generate output*/ }
		\FOR{$1\leq j \leq r$}
		\STATE $u'_j\gets \sum_{i=1}^{r+1} L'_{i,j}v_i,\quad A'_j\gets \sum_{i=1}^{r+1} R'_{i,j}B_i$
		\ENDFOR
	\end{algorithmic}
\end{algorithm}

\begin{algorithm}[h]
	\caption{$Opt(C)$}
	\label{alg:Opt}
	\begin{algorithmic}
		\STATE{{\bfseries Input:} Matrix $C\in\mathbb{R}^{(r+1)\times (r+1)}$}
		\STATE{{\bfseries Output:} Random matrices $L',R'\in\mathbb{R}^{(r+1)\times r}$, so that $L'R'^T$ is an unbiased, min-variance approximator of $C$}
		\vspace{3pt}
		\STATE{\bfseries /* Reduce to diagonal matrix $D$*/}
		\STATE $(D,U,V) \gets \mathrm{SVD}(C)$ 
		\STATE $(d_1,\ldots,d_{r+1}) \gets$ diagonal entries of $D$
		\vspace{3pt}
		\STATE{\bfseries /* Find approximator $ZZ^T$ for small $d_i$ ($i\geq m$)*/}
		\STATE $m \gets \min\{i\colon (r-i+1)d_i\leq \sum_{j=i}^r d_j\}$ 
		\STATE $s_1\gets \sum_{i=m}^{r+1} d_i,\quad k \gets r-m +1$
		\STATE $z_0 \gets \left(\sqrt{1 - \frac{d_m k} {s_1}},\ldots,\sqrt{1 - \frac{d_{r+1} k} {s_1}}\right)^T \in \R^{(k+1)\times 1}$
		\STATE $z_1,\ldots, z_k \gets$ so that $z_0,z_1,\ldots,z_k$ is an onb of $\R^{(k+1)\times 1}$
		\STATE $s\gets$ vector of $k+1$ uniformly random signs
		\STATE $Z\gets \sqrt{\frac{s_1}{k}}\cdot (s\odot z_1,\ldots,s\odot z_k)$ \COMMENT{\textit{pointwise product }$\odot$}
		\vspace{3pt}
		\STATE {\bfseries /* Initialise $L',R'$ to approximate $D$*/}
		\STATE $L',R' \gets \mathrm{diag}(\sqrt{d_1},\ldots, \sqrt{d_{m-1}},Z)$ \COMMENT{\textit{Block-diagonal}}
		\STATE{\bfseries{/*Approximate $C=UDV^T$*/}}
		\STATE $L'\gets UL',\quad R'\gets  VR'$
	\end{algorithmic}
\end{algorithm}	

	\subsection{Kronecker Triple Product (KTP)}\label{subsec:KTP}
	Finally, we present another, more exploratory algorithm approximating RTRL still following the framework from Algorithm \ref{alg:RTRLapp}.  KTP approximates $G_t$ by a sum of $r$ Kronecker-triple-products, i.e. $G'_t = \sum_{i=1}^r a_i\otimes b_i\otimes c_i$ where $a_i, c_i\in\R^{1\times n}$ and $b_i\in\R^{n\times 1}$. Before describing the remaining details of KTP, we motivate the suggested changes: On the one hand, KTP only requires memory of order $rn$ rather than $n^2$ for each batch element. On the other hand, when computing $H_tG_{t-1}$ we can write each of the addends as $H_t (a_i\otimes b_i\otimes c_i) = a_i \otimes (H_tb_i) \otimes c_i$. Computing $H_t b_i$ only\footnote{It is possible to evaluate $H_t b$ without storing $H_t$ for each batch element, see appendix~A.0.3.} takes time $n^2$, as opposed to time $n^3$ for the matrix-matrix multiplications of KF-RTRL and OK. Thus, KTP matches the memory and runtime of TBPTT.
	
	We now describe the remaining details of KTP. In order to bring $H_tG_{t-1}$ into the original format, we simply rewrite $H_t (a_i\otimes b_i\otimes c_i) = a_i \otimes (H_tb_i) \otimes c_i = a_i \otimes \overline{b}_i\otimes c_i$. To bring $F_t$ in the same format, we again make use of the fact, that $F_t$ can be factored as $F_t = h\otimes D$ where $D$ is a diagonal matrix. This allows us to easily find an optimal, unbiased rank-r approximator $D'=\sum_{i=1}^r d_i\otimes d_i^T$ of $D$, where the $d_i$ are random vectors constructed with Algorithm \ref{alg:Opt}. Note that this algorithm is similar to the original UORO approach, but uses its knowledge about $D$ in order to construct an optimal (rather than non-optimal) low-rank approximator of $D$ and in order to reduce memory requirements from $n^2$ to $n$.\\
	 All in all, we have rewritten $F'_t=\sum_{i=1}^r h\otimes d_i\otimes d_i^T$ in the desired format. It remains to mix the two addends $H_t G'_{t-1}$ and $F'_t$. To this end, we mix, for each $i$, the $i$-th summands of $H_t G'_{t-1}$ and $F'_t$ using a generalisation of the sign trick: We choose a vector of three signs $(s_1,s_2,s_1s_2)$, where $s_1,s_2$ are uniform and independent, and approximate ${a_i\otimes\overline{ b}_i\otimes c_i  + h\otimes d_i\otimes d_i^T}$ by 
	{${(a_i + s_1\cdot h)\otimes( \overline{b}_i + s_2\cdot d_i)\otimes(c_i +s_1s_2 \cdot d_i^T)}$}, so that altogether we obtain
	$$ G'_t = \sum_{i=1}^r(a_i + s_1\cdot h)\otimes(\overline{b}_i + s_2\cdot d_i)\otimes(c_i +s_1s_2 d_i^T).$$
	
	The `mixing' procedure presented above is based on heuristics. We show in the appendix~(A.0.2) that heuristics are somewhat necessary since  the concept of an `optimal' approximator is not well-defined in this case and related to NP-hard problems~\cite{hillar2013most}.
	
	\subsection{Comparison}\label{subsec:Comparison}
	When comparing different unbiased approximations of RTRL, the focus lies on comparing noise/variance of the respective approximators, since this determines the speed of convergence and the final performance. To make comparisons as fair as possible we will also consider the different runtime and memory requirements, see also Table \ref{table:cost_vert} for an overview. Here, we focus on theoretical  considerations. Experimental evaluations of the noise will be presented in the next section.
	
	 In \cite{mujika2018approximating} it was already shown that KF-RTRL has significantly less noise than UORO, so that we shall focus on KF-RTRL and OK only. 
	 
	We first compare $1$-OK to KF-RTRL, as the memory and runtime requirements for these two algorithms are asymptotically equal. The difference between the two algorithms is how they mix a sum of two Kronecker-products to obtain one Kronecker-product. From Theorem \ref{thm:opt} it is immediate that OK performs at least as well as KF-RTRL. In general, it depends on the two Kronecker-products, which need to be mixed, how much better OK performs than KF-RTRL. We show two extreme cases here - the `average' case arising during learning lies somewhere inbetween, and is assessed empirically in the next section. Suppose we need to approximate $u\otimes A + h\otimes D$ unbiasedly by a single Kronecker-product. For simplicity, let us assume that all vectors and matrices have norm 1. This makes the variance reduction technique of UORO and KF-RTRL, which is also implicitly included in the OK algorithm, unnecessary. \\
	\underline{Case 1:} $u=h$. Then, we can simply rewrite $u\otimes A + h\otimes D = u\otimes(A+D)$, which is a single Kronecker-product and doesn't need a noisy approximation. This shows that an optimal approximator like OK has variance 0. On the other hand, KF-RTRL performs the sign trick and approximates the sum either by $(u+h)\otimes(A+D)=2u\otimes(A+D)$ or by $(u-h)\otimes(A-D)=0$ and thus has variance $\Vert A+D\Vert^2$.\\ 
	\underline{Case 2:} $u\perp h$ and $A\perp D$. With the methods presented in the appendix, it can be shown that the sign-trick performed by KF-RTRL is optimal and that therefore there is no difference between KF-RTRL and OK in this case.
	
	We now inspect $r$-OK for $r>1$. In this case, OK takes $r$ times more runtime and memory than KF-RTRL. To make comparisons fair, we compare OK to running $r$ independent copies of KF-RTRL and taking their average to perform gradient descent\footnote{This reduces the noise of KF-RTRL by a factor of $r$.}, let us refer to this algorithm as $r$-KF-RTRL-AVG, or simply KF-AVG. Again, it can be deduced from Theorem \ref{thm:opt} that the noise of OK is at most that of KF-AVG. To see this, observe that KF-AVG stores $r$ Kronecker-products at each step and mixes them with $h_t\otimes D_t$ to obtain another $r$ Kronecker-factors for the next step.\\ 
	For $r>1$ there is an additional, important phenomenon improving OK over KF-AVG, which we illustrate now. Suppose $r=2$ and we want to approximate $u_1\otimes A_1 + u_2\otimes A_2 + h\otimes D$ unbiasedly by a sum of two Kronecker-products. Assume that $u_1,u_2,h$ and $A_1,A_2,D$ respectively are pairwise orthogonal\footnote{In fact, any sum can be rewritten in such a way. This is equivalent to SVD.} and consider the case where one of the summands is larger than the other two, say $\Vert u_1\otimes A_1\Vert = 10$ and $\Vert u_2\otimes A_2\Vert = \Vert h\otimes D\Vert = 1$. Then, the optimal approximator OK will keep $u_1\otimes A_1$ fixed and mix only the other two summands creating noise of order 1. More na\"ive approaches, including the sign trick of KF-AVG, mixes all factors and create noise of order 10/$r$, $r=2$. This phenomenon of keeping important parts of the gradient and only mixing less important parts to reduce the noise becomes important and appears more frequently as $r$ gets larger, see also Figure \ref{fig:noise_ptb_comparison} for experimental evidence.
	
	\begin{table}[t]
		\caption{Computational costs for different algorithms, measured per batch element and parameter update. These values reflect the cost in an actual implementation. The additional cost of storing the model (memory $n^2$) does not scale with the batch size and is therefore negligible when training with large mini-batches.\\ 
		Dashed horizontal lines group algorithms with comparable costs. $r$ is a parameter of the algorithms, $T$ is the truncation horizon of TBPTT. See Section \ref{sec:rtrlapp} for details.}
		\label{table:cost_vert}
		\vskip 0.15in
		\begin{center}
			\begin{small}
				\begin{sc}
					\begin{tabular}{lcc}
						\toprule
						&Memory & Runtime \\
						\midrule
						RTRL & $n^3$ & $n^4$\\  \hdashline
						$r$-OK    & $rn^2$ & $rn^3$\\
						$r$-KF-RTRL-AVG    & $rn^2$ & $rn^3$ \\ \hdashline
						UORO & $n^2$ & $n^2$ \\ \hdashline
						$r$-KTP & $rn$ & $rn^2$ \\
						TBPTT-$T$    & $Tn$ & $Tn^2$ \\
						\bottomrule
					\end{tabular}
				\end{sc}
			\end{small}
		\end{center}
		\vskip -0.1in
	\end{table}
	
	\section{Experiments}\label{sec:exp}
	Here, we empirically analyze the advantage of using the optimal approximation from OK as opposed to the sign trick from KF-RTRL. Moreover, we compare OK to TBPTT, showing that the noise in OK is so small that it does not hinder its learning performance. We posit that this is due to the noise of OK being smaller than that of Stochastic Gradient Descent~\cite{Robbins1951sgd}. This is independent of the batch size $b$, as both sources of noise are divided by the same factor $b$. Figure A.1 in the appendix, which is similar to Figure~\ref{fig:ptb_comparison} but with larger batches, portrays this point. 

	 Following~\cite{mujika2018approximating}, we assess the learning performance of OK in two tasks. The first, termed Copy task, is a synthetic binary string memorization task which evaluates the RNN's ability to store information and learn long-term dependencies. The second, character-level language modeling on the Penn TreeBank dataset (CHAR-PTB), is a complex real-world task commonly used to assess the capabilities of RNNs. We compare the performance of OK to KF-RTRL and TBPTT based on `data time', i.e. on how much data the algorithm is given. Moreover, we perform an experiment analyzing the variance of OK and KF-RTRL by comparing them to the exact gradients given by RTRL. Lastly, we measure the performance of our second algorithm KTP on the Copy task and show that it can learn moderate time dependencies. For all experiments, we use a single-layer Recurrent Highway Network~\cite{zilly2017recurrent}\footnote{For implementation simplicity, we replace $\tanh(x)$ by $2*\textrm{sigmoid}(x)-1$ as the non-linearity function. These functions have similar properties, so this should not have any significant effect on learning.}.	

	\subsection{Comparisons between OK, KF-RTRL and TBPTT}
		\subsubsection{Copy task}\label{subsec:copy_task}
			For the Copy task, a binary string of length $T$ is presented sequentially to an RNN. Once the full string is presented, the RNN should reconstruct the original string without any extra information (example for a sequence of length $5$: input \#01101****** and target output ******\#01101). The results are shown in Figure~\ref{fig:string_memorization_comparison}. OK outperforms KF-RTRL when making a comparison that equates the memory and runtime requirements between the two approaches (see Section~\ref{subsec:Comparison}). Furthermore, by exploiting the online updates, OK also outperforms TBPTT when giving both algorithms the same network and batch sizes. 
			\MG{perhaps the point in the last two sentences should be in Comparisons and not here}\FB{i would leave it here}\FB{I would make the previous sentences less extreme, maybe: When comparing KF-RTRL with TBPTT it is important to note that TBPTT could be run longer and with larger batch-size due to lower memory and runtime requirements than OK. The comparison is fair in the sense that it gives both algorithms access to the same amount of data. Remarkably, in this setting OK outperforms TBPTT indicating that it reduces the noise enough to take advantage of its online updates over TBPTT.}
			It is important to note that the runtime and memory advantage of TBPTT imply that it could be run on a larger network and for longer. The comparison done here is fair in the sense of giving both algorithms the same amount of data and assesses whether the noise of OK has been reduced to the point where it does not harm learning. 

			\FB{maybe start with: We now describe the details of our implementation.}We now describe the details of our implementation. As in~\cite{mujika2018approximating}, we use curriculum learning and start with $T=1$, increasing $T$ by one when the RNN error drops below $0.15$ bits/char. After each sequence, the hidden states are reset to zero. To improve performance, the length of the sequence is sampled uniformly at random from $T-5$ to $T$. This forces the network to learn a general algorithm, as opposed to one suited only for sequences of length $T$. We use a RHN with $128$ units and a batch size of $16$. We optimize the log-likelihood using the Adam optimizer~\cite{kingma2014adam} with default Tensorflow~\cite{abadi2016tensorflow} parameters, $\beta_1 = 0.9$ and $\beta_2 = 0.999$. For each model, we pick the best learning rate from $\{10^{-2.5}, 10^{-3}, 10^{-3.5}, 10^{-4}\}$. We repeat each experiment $5$ times. 

	\begin{figure}[t]
		\vskip 0.2in
		\begin{center}
			\centerline{\includegraphics[width=0.8\columnwidth]{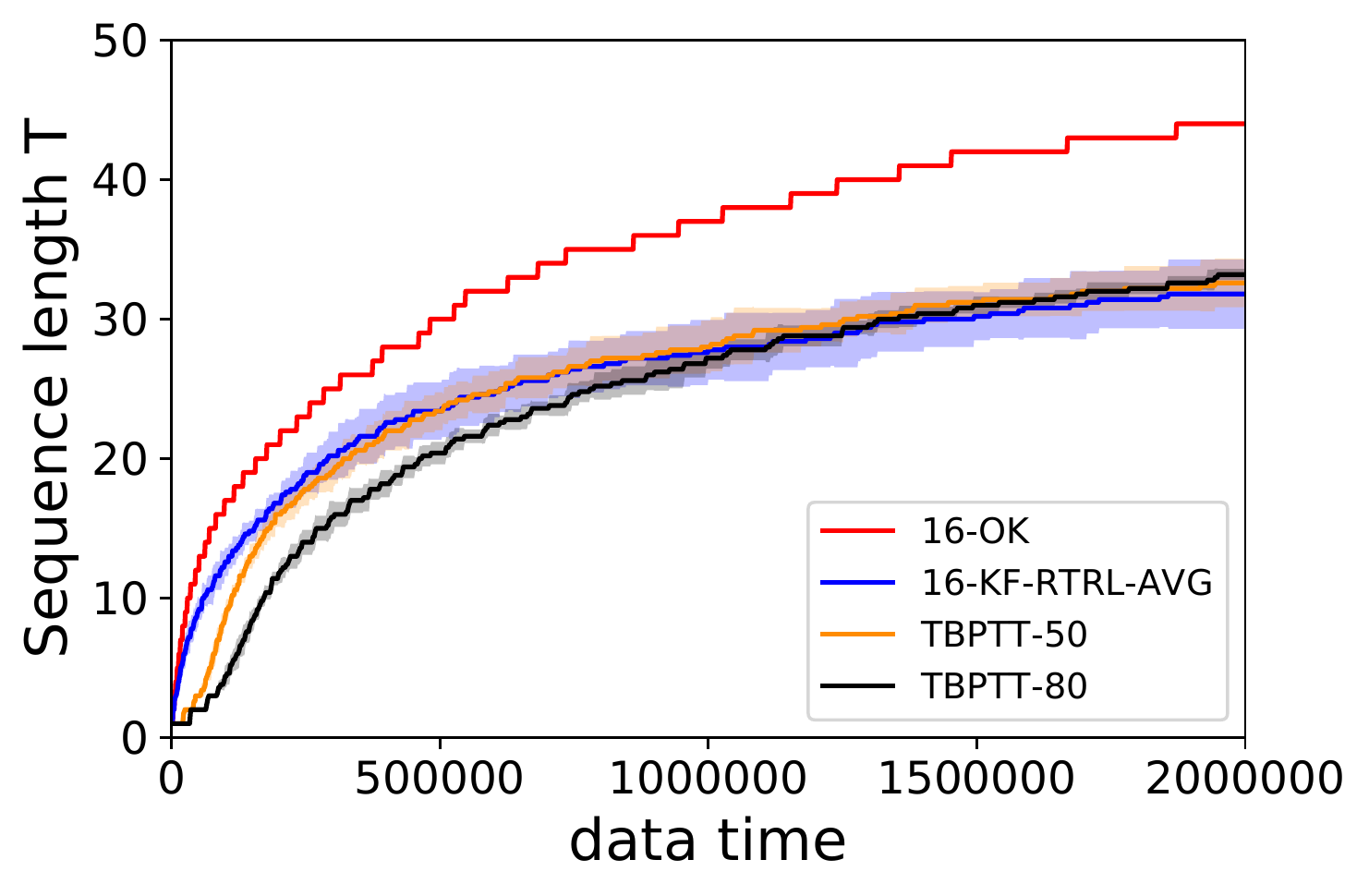}}
			\caption{Copy task. We plot the mean and standard deviation (shaded area) over $5$ trials. $16$-OK learns sequences on average up to $42$, $16$-KF-RTRL-AVG up to $32$, TBPTT-$50$ and $80$ up to $33$. We trained an RHN with $128$ units for all models.}
			\label{fig:string_memorization_comparison}
		\end{center}
		\vskip -0.2in
	\end{figure}

	\subsubsection{CHAR-PTB on the Penn Treebank dataset}
	
	For the CHAR-PTB task, the network receives a text character by character, and at each time step it must predict the next character. This is a standard, challenging test for RNNs which requires capturing long- and short-term dependencies. It is highly stochastic, as there are many potential continuations for most input sequences. Figure~\ref{fig:ptb_comparison} and Table~\ref{table:ptb_val_test} show the results. $8$-OK outperforms $8$-KF-RTRL-AVG, and matches the performance of TBPTT-25. In fact, $8$-OK even takes advantage of its online updates to achieve faster convergence when compared to TBPTT-25. The advantage observed in Figure~\ref{fig:ptb_comparison} is even larger when using longer truncation horizons as suggested by Figure~\ref{fig:string_memorization_comparison}. This fact showcases the strength of performing online updates as in RTRL as opposed to having an update lock as in TBPTT.

	For this experiment we use the Penn TreeBank~\cite{marcus1993building} dataset, a collection of Wall Street Journal articles commonly used for training character level models. We split the data following~\cite{mikolov2012subword}. In addition, we reset the hidden state to zero with a probability of $0.01$ at every step~\cite{melis2017state}. The experimental setup is the same as in Section~\ref{subsec:copy_task}, except the RHN has $256$ units and the batch size is $32$. The learning rates are chosen in the same range. 

	\begin{figure}[ht]
		\vskip 0.2in
		\begin{center}
			\centerline{\includegraphics[width=0.8\columnwidth]{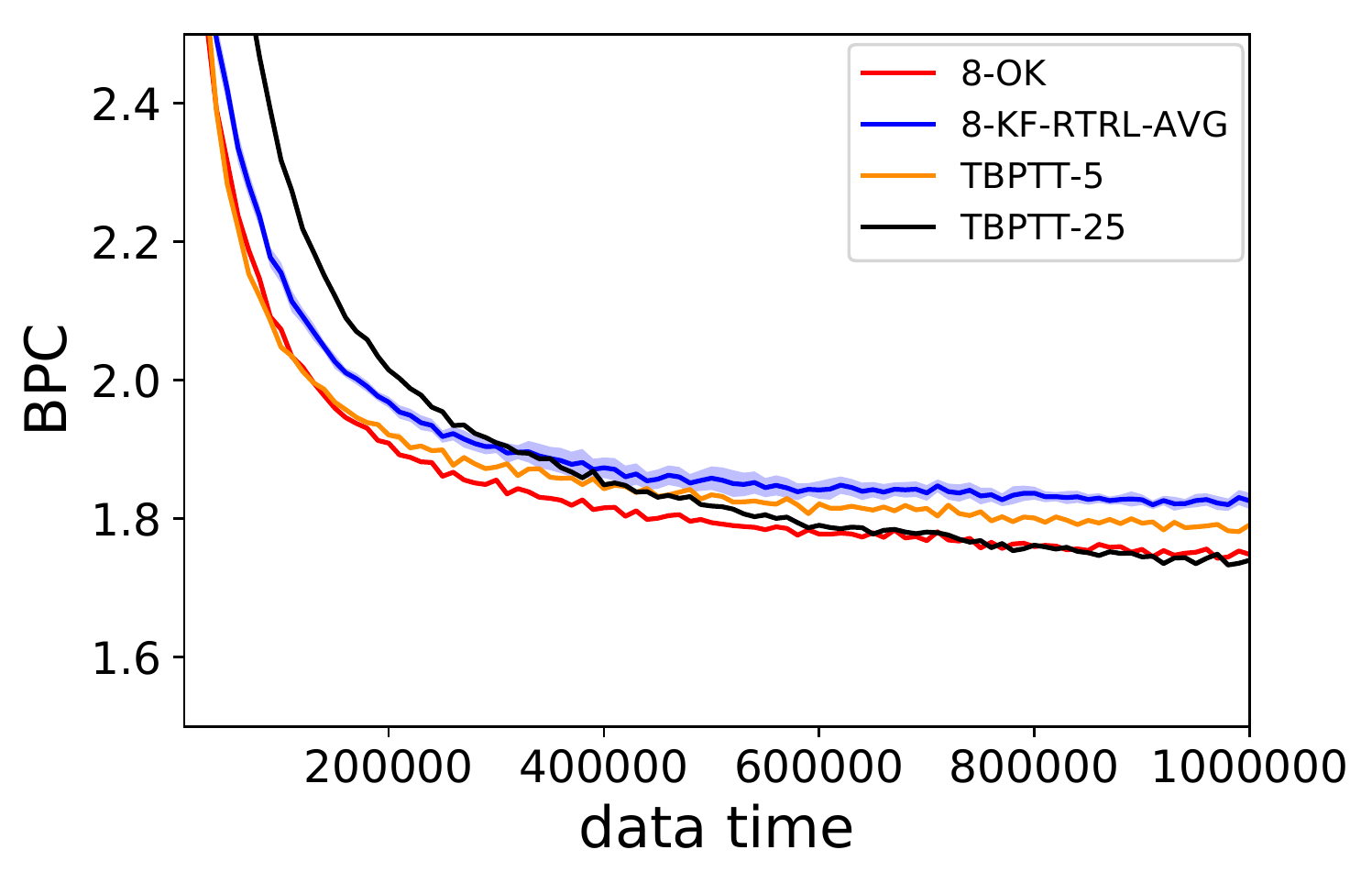}}
			\caption{Validation performance on Penn TreeBank in bits per character (BPC). $8$-OK matches the performance of TBPTT-$25$. We trained a RHN with $256$ units for all models. Table~\ref{table:ptb_val_test} summarizes the performances.
			}
			\label{fig:ptb_comparison}
		\end{center}
		\vskip -0.2in
	\end{figure}

	\begin{table}[ht]
		\caption{Results on Penn TreeBank. Merity et al.~\yrcite{merity2018analysis} is the current state of the art. Standard deviations are smaller than $0.01$.}
		\label{table:ptb_val_test}
		\vskip 0.15in
		\begin{center}
			\begin{small}
				\begin{sc}
					\begin{tabular}{lcccr}
						\toprule
						Name & Validation & Test & \#params \\
						\midrule
						$8$-KF-RTRL-AVG    & $1.82$ & $1.77$ & $133K$ \\
						$8$-OK    & $1.74$ & $1.69$ & $133K$ \\
						TBPTT-$5$    & $1.78$ & $1.73$ & $133K$ \\
						TBPTT-$25$    & $1.73$ & $1.69$ & $133K$ \\
						Merity et al.~\yrcite{merity2018analysis}    & $-$ & $1.18$ & $13.8M$ \\
						\bottomrule
					\end{tabular}
				\end{sc}
			\end{small}
		\end{center}
		\vskip -0.1in
	\end{table}

	\subsection{Empirical Exploration of Noise}\label{subsec:empirical_noise}
		Here, we empirically evaluate how the noise evolves over time. We report the cosine between the true gradient and the approximated one. The results for an untrained network are given in the appendix (Figure A.2). There, already $2$-OK achieves a cosine of almost exactly $1$. However, the most interesting behavior arises later in training. Figure~\ref{fig:noise_ptb_comparison} shows that, after a million steps of training on CHAR-PTB, the cosine is much smaller for $8$-KF-RTRL-AVG than for $8$-OK. 

	For this experiment, we train a RHN with $256$ units on CHAR-PTB for $1$ million steps. Then, we freeze the weights of the network and compute the angle $\phi$ between the gradient estimates provided by OK and KF-RTRL and the true RTRL gradient for $1000$ steps. We plot the mean and standard deviation of $20$ repetitions of each experiment. In the appendix, Figure A.3 provides similar experiments for the Copy task.

	\begin{figure}[ht]
		\vskip 0.2in
		\begin{center}
			\centerline{\includegraphics[width=0.8\columnwidth]{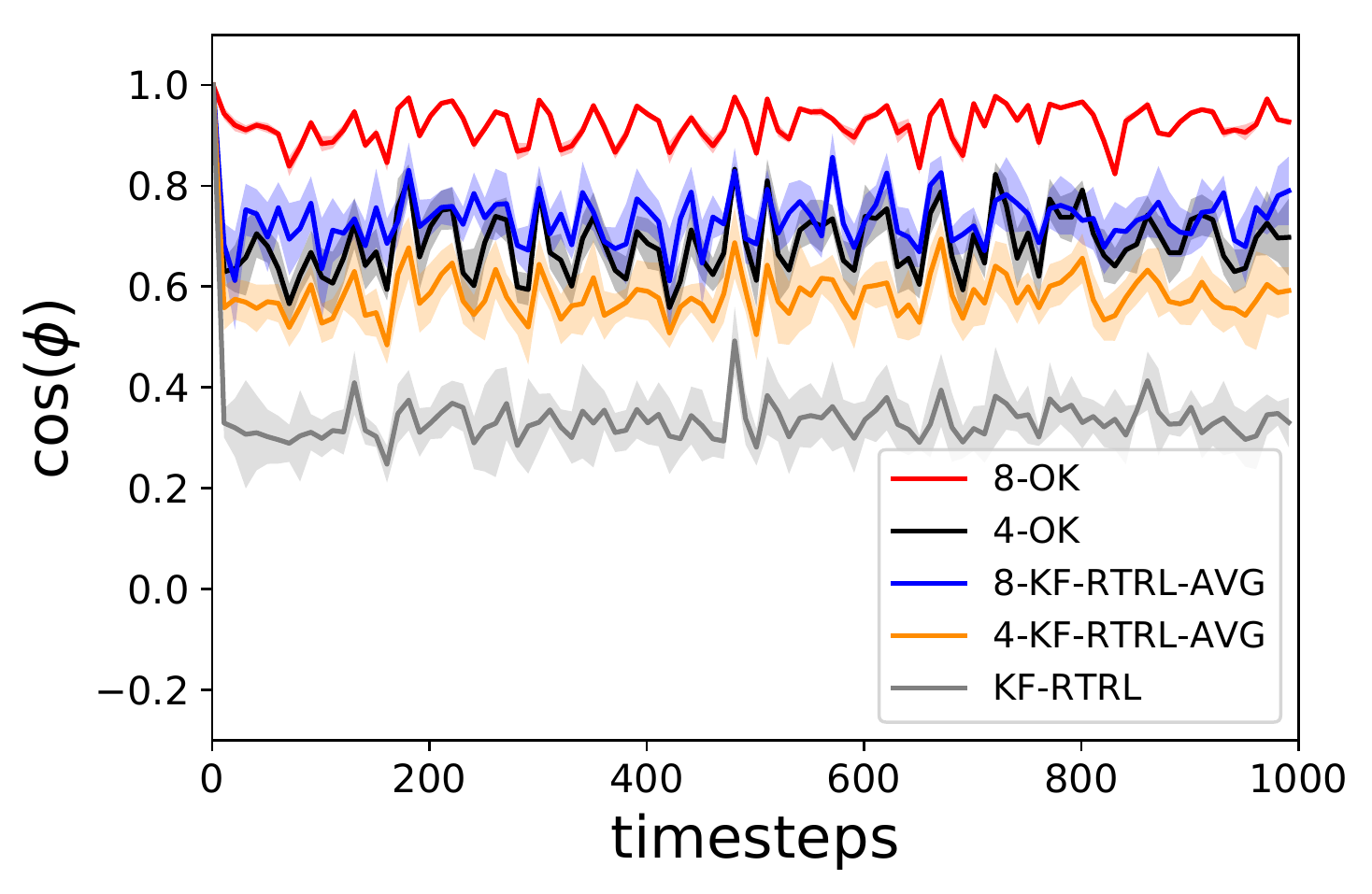}}
			\caption{Variance analysis in a RHN trained for $1$ million steps on CHAR-PTB. We plot the cosine of the angle between the approximated and the true value of $\frac{dL}{d\theta}$. A cosine of $1$ implies that the approximation and the true value are aligned, whereas a random vector has an expected cosine of $0$.
			}
			\label{fig:noise_ptb_comparison}
		\end{center}
		\vskip -0.2in
	\end{figure}

	\subsection{Kronecker Triple Product}
		We now analyze the performance of KTP. While KTP has the same memory and runtime requirements as TBPTT, this comes at the cost of extra noise. KTP possesses two sources of noise (see Section~\ref{subsec:KTP}). The first is added by a low-rank approximation $D'$ of the diagonal $D$, the second is added in the mixing procedure\FB{replace 'by the sign trick' by 'in the mixing procedure'?}. Here, we show experiments indicating\FB{or show experiments indicating} that the first noise source is not significant by artificially introducing it to KF-RTRL. This is done by unbiasedly approximating $F_t=h\otimes D$ by $h\otimes D'$, where $D'$ is as in Section~\ref{subsec:KTP}. The rest of the KF-RTRL algorithm remains as usual. We term this adapted version KF-RTRL-$r$-APPROX when $D$ is approximated by a rank $r$ matrix. Figure~\ref{fig:cp_diag_approx_A_full_rank} shows that KF-RTRL-$16$-APPROX performs almost as well as the original KF-RTRL, suggesting that the noise added in the mixing procedure is what hurts KTP the most.\FB{maybe add conclusion, sth like: This indicates that noise added in the mixing procedure is what hinders learning of KTP most.}

For the experiment, we use the same setup as in Section~\ref{subsec:copy_task} except that the batch sizes used were $256$. We plot the mean for $5$ repetitions of the experiment.

		\begin{figure}[ht]
			\vskip 0.2in
			\begin{center}
				\centerline{\includegraphics[width=0.8\columnwidth]{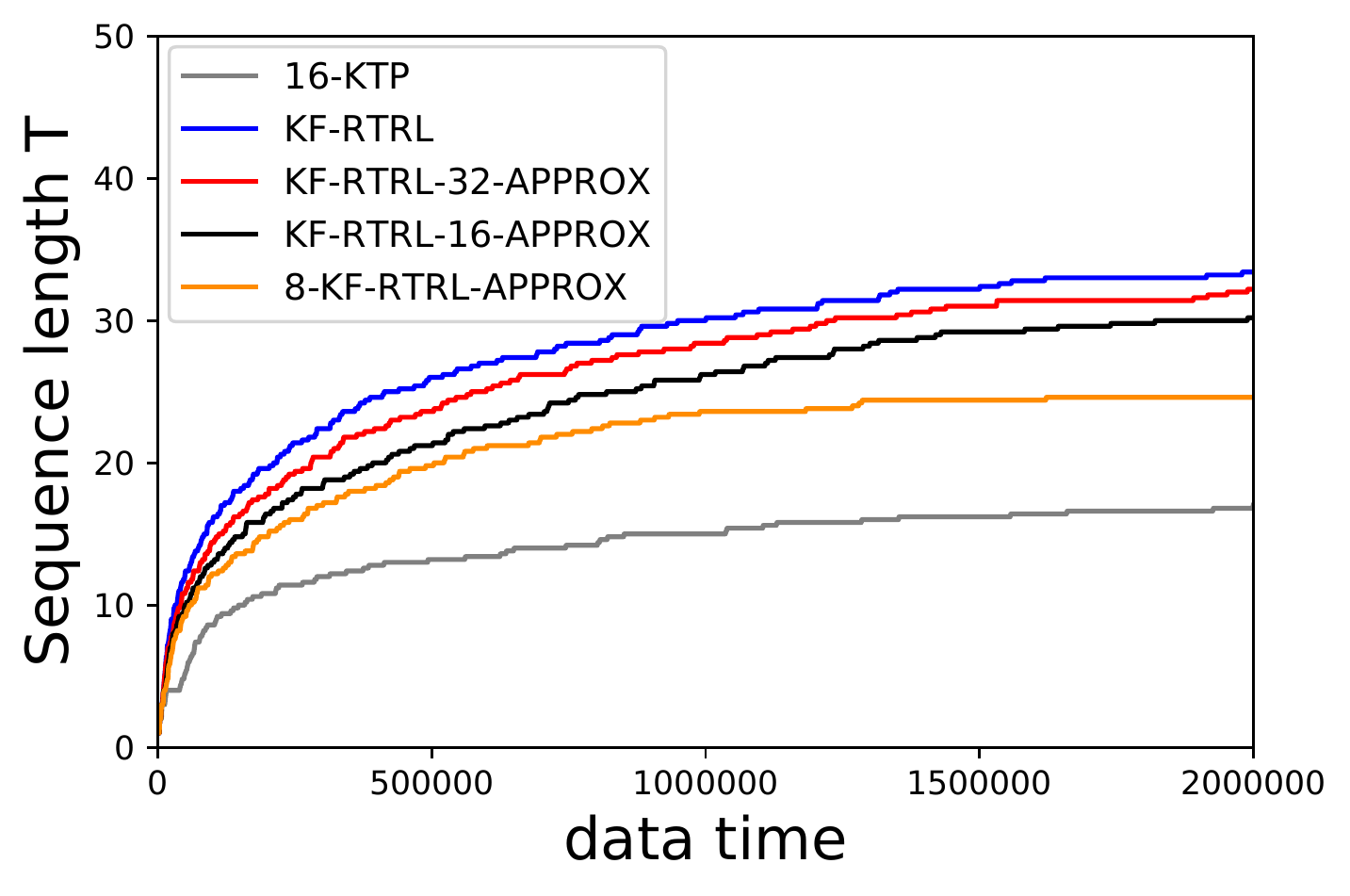}}
				\caption{KTP performance on the Copy task. From top to bottom as in the legend, the learned sequence lengths are: $17$, $33$, $32$, $30$, $25$. Standard deviations were around $3$ for all algorithms.}
				\label{fig:cp_diag_approx_A_full_rank}
			\end{center}
			\vskip -0.2in
		\end{figure}

	\section{Conclusions}
	We presented two new algorithms, OK and KTP, for training RNNs. Both are unbiased approximations of Real Time Recurrent Learning (RTRL), an online alternative to Truncated Backpropagation through Time (TBPTT) giving untruncated gradients. For OK, we do not only show that it has less variance than previous approximations, but show that our approximation is in fact optimal for the class of Kronecker-Sum approximations, which includes all previously published approaches. We empirically show that this improvement makes the noise of OK negligible, which distinguishes OK from previous approximations of RTRL. This is evaluated on the standard benchmark of Penn TreeBank (PTB) character-level modeling, where OK matches the performance of TBPTT. In the case of a synthetic string-memorization task, OK can exploit frequent online parameter updates to outperform TBPTT. Our second algorithm, KTP, is more exploratory and paves the way towards more memory and runtime efficient approximations of RTRL. Its computational cost matches that of TBPTT and we show that it can learn moderate time dependencies. Reducing the noise of KTP provides an interesting problem for further research.
	
	Our theoretical optimality result shows that, if the noise of RTRL is to be reduced further, new classes of approximations need to be explored. Moreover, the result can be extended to test the theoretical limitations of new approximations of RTRL which obey a similar structure as the Kronecker-Sum.
	 We also presented a more general algorithm to construct unbiased, low-rank approximators of matrices with minimum achievable variance. This might be useful in other areas of machine learning which rely on unbiased gradients.
	 
	Conceptually, we explore an alternative to TBPTT. We believe that this is a necessary step towards learning long-term dependencies and for making full use of the architectural developments that have recently advanced RNNs.
	While RTRL itself is infeasible due to large computational costs, our results indicate that it is possible to reduce its memory and runtime requirements by a factor of $n$ while keeping the noise small enough to not harm learning. Further improvements in this direction would already make RTRL a viable alternative to TBPTT and impact modeling data with inherent long-term dependencies.

	\section*{Acknowledgements}
	We would like to thank Florian Meier and Pascal Su for helpful discussions and valuable comments on the presentation of this work. We also thank the anonymous reviewers for their helpful feedback.
	
	Frederik Benzing was supported by grant no. 200021\_169242 of the Swiss National Science Foundation.
	Marcelo Matheus Gauy was supported by CNPq grant no. 248952/2013-7. 
	Asier Mujika was  supported by grant no. CRSII5\_173721 of the Swiss National Science Foundation.

\bibliography{OK}
\bibliographystyle{icml2019}


\section*{APPENDIX}
We begin this appendix by stating a precise condition on the subclass of RNNs to which the algorithms OK and KF-RTRL can be applied (Section \ref{sec:condition}) and show that in a setting as given by the algorithm KTP, the concept of a minimum-variance approximator is not well defined (Section \ref{sec:ill}). We also give details on how the memory requirement of KTP can be kept at $O(n)$ (Section \ref{sec:KTPmem}).

In the remaining two sections, we prove Theorems 1 and 2 from the paper (Section \ref{sec:theory}) and provide additional experiments (Section \ref{sec:exp}). 

\subsection{Subclass of RNNs for OK and KF-RTRL}\label{sec:condition}
Recall, that similarly to KF-RTRL~\cite{mujika2018approximating}, we restrict our attention to RNNs for which the term $F_t$ (see description of RTRL in the paper) can be factored as $F_t = h_t\otimes D_t$. We restate the condition given in~\cite{mujika2018approximating}:
\begin{lemma}\label{lem:applicable}
	Assume that the learnable parameters $\theta$ are a set of matrices $W^1,...,W^r$, let $\hat{h}_{t-1}$ be the hidden state $h_{t-1}$ concatenated with the input $x_t$ and let $z^k = \hat{h}_{t-1}W^k$ for $k = 1, . . . , r$. Assume that $h_t$ is obtained by point-wise operations over the $z^k$'s, that is, $(h_t)_j = f(z^1_j, . . . , z^r_j)$. Let $D^k \in \R^{n\times n}$ be the diagonal matrix defined by $D^j_{kk} = \frac{\partial{(h_t)_j}}{\partial{z^k_j}}$, and let $D =(D^1\vert\ldots\vert D^r)$. Then, it holds that $\frac{\partial{h_t}}{\partial{\theta}} = \hat{h}_{t-1}\otimes D$.
\end{lemma} 
We refer the reader to~\cite{mujika2018approximating} for the simple proof. There, it is  also shown that this class of RNNs includes standard RNNs and the popular LSTM and RHN architectures.

\subsection{No Optimal Approximation for 3-Tensors}\label{sec:ill}
Here, we show that the concept of a minimum-variance approximator is ill-defined for a situation as encountered by KTP. We also explain how similar problems are related to NP-hardness. 

For the next lemma, we slightly adapt an example from~\cite{hillar2013most} based on an exercise in~\cite{Knuth97}. We first recall some notions related to 3-tensors. For $a,b,c\in\R^{n}\setminus\{0\}$, we call $a\otimes b\otimes c$ a rank-1 tensor. In general, the rank of a tensor $X$ is the minimum number $k$, so that $X$ can be written as the sum of $k$ rank-1 tensors. The following is related to the fact that the set of rank-2 tensors is not closed.
\begin{lemma}\label{lem:ill}
	For $i=1,2,3$, let $x_i,y_i\in \R^n$ so that the pairs $x_i,y_i$ are linearly independent and define ${X= x_1\otimes x_2 \otimes y_3 + x_1 \otimes y_2 \otimes x_3 + y_1\otimes x_2 \otimes x_2}$.\\
	Then, there are rank-2 approximators of $X$ with arbitrarily small variance, but no rank-2 approximator of variance 0. Thus, the concept of a `minimum-variance' rank-2 approximator of $X$ is ill-defined.
\end{lemma}
\begin{proof}
	The statement that there is no rank-2 approximator with variance 0 is equivalent to $X$ having rank larger than $2$, the details of which we leave to the reader.
	
	Now, let $s\in\{\pm1\}$ be a uniformly random sign and for each positive integer $n$ define a random variable 
	\begin{eqnarray*}
	X_n =&& x_1\otimes x_2 \otimes \left(y_3 - s\cdot nx_3\right) \\
	&+& \left(s\cdot x_1 + \frac{1}{n}y_1\right)\otimes\left(x_2 + s\cdot \frac{1}{n}y_2\right)\otimes nx_3
	\end{eqnarray*}
	A simple calculation shows
	$
	X_n = X +s\cdot \frac{1}{n}\left(y_1\otimes y_2 \otimes x_3\right).
	$
	From this we conclude $\E[X_n] = X$ and $\Var[X_n]\rightarrow 0$ as $n\rightarrow\infty$, finishing the proof of the lemma.
\end{proof}
In addition to the concept of a minimum-variance approximator not being well defined, we note that finding `good' approximators (which might still be possible given the above lemma) seems to be closely related to finding the rank of a 3-tensor, which is, like many other problems for 3-tensors~\cite{hillar2013most}, NP-hard.

\subsection{Memory of KTP}\label{sec:KTPmem}
Here, we describe how the memory requirement of KTP can be kept at $O(n)$ despite the need of calculating $H_t b$ (see description of KTP in the paper, $H_t\in\R^{n\times n},b\in\R^{n\times 1})$. One way to do this was already used in~\cite{tallec2017unbiased}. Recall $h_t =f(x_t,h_{t-1},\theta)$ and $H_t=\frac{\partial h_t}{\partial h_{t-1}}$. Therefore, $H_t b$ is a directional derivative of $h_t$ in the direction of $b$, which implies
\begin{eqnarray}
H_t b = \lim_{\epsilon\rightarrow 0} \frac{f(x_t, h_{t-1}+\epsilon b,\theta) - f(x_t, h_{t-1},\theta)}{\epsilon \lVert b\rVert}.
\end{eqnarray}
To evaluate $H_t b$ it therefore suffices to choose a small $\epsilon$ and evaluate the expression above. The expression above can be calculated together with the forward pass of the RNN, so that no additional batch-memory is needed.

\section{Proofs}\label{sec:theory}
In this section, we prove Theorems 1 and 2 from the main paper. Their statements and the related algorithms are restated below for convenience. 

The outline of this section is as follows. 
We start  in Section \ref{sec:prelim} by introducing some notation and reviewing the concept of a Singular Value Decomposition along with some of its properties, which will be useful later.
In Section \ref{sec:opt} we prove Theorem \ref{thm:opt} assuming correctness of Algortihm \ref{alg:Opt} and Theorem \ref{thm:mat}. These are then jointly proved in
 Section \ref{sec:mat}.

\begin{thm}
	\label{thm:opt}
	Let $G$ be an $(r+1)$-Kronecker-Sum and let $G'$ be the random $r$-Kronecker-Sum constructed by OK. Then $G'$ unbiasedly approximates $G$. Moreover, for any random $r$-Kronecker-Sum $Y$ of the same format as $G'$ which satisfies $\E[Y] = G$,  it holds that $\Var[Y]\ge \Var[G']$.
\end{thm}

\begin{thm}
	\label{thm:mat}
	Given $C\in\R^{m\times n}$ and  $r\le \min\{m,n\}$, one can (explicitly) construct an unbiased approximator $C'$ of $C$, so that $C'$ always has rank at most $r$, and so that $C'$ has minimal variance among all such unbiased, low-rank approximators. This can be achieved asymptotically in the same runtime as computing the SVD of $C$.
\end{thm}

\begin{algorithm}
	\caption{The OK approximation}
	\label{alg:OK}
	\begin{algorithmic}
		\STATE{{\bfseries Input:} Vectors $u_1,\ldots, u_{r+1}$ and matrices $A_1,\ldots,A_{r+1}$}
		\STATE{{\bfseries Output:} Random vectors $u'_1,\ldots, u'_r$ and matrices $A'_1\ldots A'_r$, such that $\sum_{i=1}^r u'_i\otimes A'_i$ is an unbiased, minimum-variance approximator of $\sum_{i=1}^{r+1} u_i\otimes A_i$}
		\vspace{3pt}
		\STATE{\bfseries /*Rewrite in terms of orthonormal basis (onb)*/}
		\STATE{$v_1,\ldots,v_{r+1}\gets$ onb of $\text{span}\{u_1,\ldots,u_{r+1}\}$}
		\STATE{$B_1,\ldots,B_{r+1}\gets$ onb $\text{span}\{A_1,\ldots,A_{r+1}\}$}
		\FOR{$1\leq i,j \leq r+1$}
		\STATE{$L_{i,j}\gets \langle v_i,u_j\rangle,\quad R_{i,j}\gets \langle B_i,A_j\rangle$}
		\ENDFOR
		\vspace{3pt}
		\STATE{\bfseries /*Find optimal rank $r$ approximation of matrix $C$ */ }
		\STATE{$C\gets  LR^T$}
		\STATE{$(L',R')\gets Opt(C)$} \COMMENT{see Algorithm \ref{alg:Opt} for $Opt(\cdot)$}
		\vspace{3pt}
		\STATE{\bfseries /*Generate output*/ }
		\FOR{$1\leq j \leq r$}
		\STATE $u'_j\gets \sum_{i=1}^{r+1} L'_{i,j}v_i,\quad A'_j\gets \sum_{i=1}^{r+1} R'_{i,j}B_i$
		\ENDFOR
	\end{algorithmic}
\end{algorithm}

\begin{algorithm}
	\caption{$Opt(C)$}
	\label{alg:Opt}
	\begin{algorithmic}
		\STATE{{\bfseries Input:} Matrix $C\in\mathbb{R}^{(r+1)\times (r+1)}$}
		\STATE{{\bfseries Output:} Random matrices $L',R'\in\mathbb{R}^{(r+1)\times r}$, so that $L'R'^T$ is an unbiased, min-variance approximator of $C$}
		\vspace{3pt}
		\STATE{\bfseries /* Reduce to diagonal matrix $D$*/}
		\STATE $(D,U,V) \gets \mathrm{SVD}(C)$ 
		\STATE $(d_1,\ldots,d_{r+1}) \gets$ diagonal entries of $D$
		\vspace{3pt}
		\STATE{\bfseries /* Find approximator $ZZ^T$ for small $d_i$ ($i\geq m$)*/}
		\STATE $m \gets \min\{i\colon (r-i+1)d_i\leq \sum_{j=i}^r d_j\}$ 
		\STATE $s_1\gets \sum_{i=m}^{r+1} d_i,\quad k \gets r-m +1$
		\STATE $z_0 \gets \left(\sqrt{1 - \frac{d_m k} {s_1}},\ldots,\sqrt{1 - \frac{d_{r+1} k} {s_1}}\right)^T \in \R^{(k+1)\times 1}$
		\STATE $z_1,\ldots, z_k \gets$ so that $z_0,z_1,\ldots,z_k$ is an onb of $\R^{(k+1)\times 1}$
		\STATE $s\gets$ vector of $k+1$ uniformly random signs
		\STATE $Z\gets \sqrt{\frac{s_1}{k}}\cdot (s\odot z_1,\ldots,s\odot z_k)$ \COMMENT{\textit{pointwise product }$\odot$}
		\vspace{3pt}
		\STATE {\bfseries /* Initialise $L',R'$ to approximate $D$*/}
		\STATE $L',R' \gets \mathrm{diag}(\sqrt{d_1},\ldots, \sqrt{d_{m-1}},Z)$ \COMMENT{\textit{Block-diagonal}}
		\STATE{\bfseries{/*Approximate $C=UDV^T$*/}}
		\STATE $L'\gets UL',\quad R'\gets  VR'$
	\end{algorithmic}
\end{algorithm}	

\subsection{Preliminaries}\label{sec:prelim}
\subsubsection{Notation}
We denote matrices by upper case letters, e.g. $C\in\mathbb{R}^{m\times n}$, and denote their entries by indexing this letter, e.g. $C_{i,j}$ where $1\leq i \leq m$ and $1\leq j \leq n$. For vectors $s,z_1,\ldots,z_n \in \mathbb{R}^{n\times 1}$, we denote by $s\odot z_i$ the pointwise product and by $Z=(z_1,\ldots,z_r)\in\mathbb{R}^{n\times r}$ the matrix whose $i$-th column is $z_i$. We write $\Id_n$ for the identity matrix of dimension $n$.

For a random variable $X'\in\mathbb{R}^{n\times m}$ and some fixed value $X\in\mathbb{R}^{n\times m}$, we say that $X'$ is an \textit{unbiased approximator} of $X$, if $\E[X'] = X$. We further call $X'$ a \textit{rank-$r$ approximator}, if $X$ always (with probability 1) has rank \textit{at most} $r$. We will usually name random variables by adding a `` ' '' to the deterministic quantity they represent. The variance of $X'$ is defined as $\Var[X'] = \E[\lVert X' -\E[X']\rVert^2]$, where we use the Frobenius norm and the corresponding inner product $\langle X_1,X_2\rangle = \Tr(X_1^TX_2)$ throughout.

\subsubsection{Singular Value Decomposition}
The Singular Value Decomposition (SVD) is a standard tool from Linear Algebra, and has countless applications in and outside Machine Learning. We refer to the textbook~\cite{golub1996matrix} for an introduction and~\cite{cline2006computation} for a review of algorithms to compute the SVD. 

To simplify notation, we restrict our attention to square matrices. The concepts straightforwardly generalise to arbitrary matrices, we refer to the above mentioned textbook.

For a matrix $C\in\R^{n\times n}$, the Singular Value Decomposition (SVD) of $C$ is a triple of matrices $U,V,D\in\R^{n\times n}$ satisfying $C=UDV^T$, so that $U,V$ are orthogonal and $D=\diag(d_1,\ldots,d_n)$ is a diagonal matrix with non-negative, non-decreasing entries. The existence of a SVD is a standard result in Linear Algebra.

The values $d_i$ are refered to as \textit{singular values} of $C$ and are uniquely determined by $C$. In fact they are the square-roots of the eigenvalues of $CC^T$. The number of non-zero singular values of $C$ equals the rank of $C$.
The columns of $U,V$ are called left, respectively right, \textit{singular vectors}. They correspond to eigen-bases of the matrices $CC^T$ and $C^TC$, respectively.   
The singular vectors are uniquely determined if and only if the singular values are pairwise distinct. If a singular value $d_i$ appears more than once, the corresponding singular vectors form an orthonormal basis of a subspace uniquely determined by $C$ and $d_i$ (corresponding to an eigen-space of $CC^T$ or $C^TC$).

One of the important applications of the SVD is the following result, known as the Eckart-Young Theorem~\cite{Eckart1936}.
\begin{thm}[Eckart-Young Theorem]
	Let $C\in\R^{n\times n}$ be a matrix with singular values $d_1,\ldots, d_n$ and let $X\in\R^{n\times n}$ be a fixed (non-random) matrix of rank at most $r$. \\
	Then, $\lVert C-X\rVert^2 \geq \sum_{i=r+1}^n d_i^2$ and equality is achieved if and only if $X$ is of the form ${X=U\cdot \diag(d_1,\ldots,d_r,0\ldots,0)\cdot V^T}$ for an arbirarty singular value decomposition $C=UDV^T$ of $C$.
\end{thm}

Noting that every SVD of the identity matrix $\Id_n$ is of the form $\Id_n = U \cdot\Id_n\cdot U^T$ for some orthogonal matrix $U$, we can deduce the following observation.

\begin{obs}\label{obs}
	Let $X\in\R^{n\times n}$ be a fixed (non-random) matrix of rank at most $r$. 
	Then $\lVert X-\Id_n\rVert^2 \geq n-r$ and equality is achieved if and only if $X$ is of the form $X= \sum_{i=1}^r u_iu_i^T$, where the $u_i$ are orthonormal vectors in $\R^{n\times 1}$. 
\end{obs}

\subsection{Proof of Theorem \ref{thm:opt}}\label{sec:opt}
Let us first restate the objective encountered by OK. We are given vectors $u_1,\ldots, u_{r+1}\in\mathbb{R}^{1\times n}$ and matrices $A_1,\ldots,A_{r+1}\in \mathbb{R}^{n\times n}$ and we want to construct random vectors $u'_1,\ldots,u'_r$ and matrices $A'_1,\ldots,A'_r$ such that the $r$-Kronecker-Sum $G' = \sum_{i=1}^r{u'_i\otimes A'_i}$ is an unbiased approximator of the $(r+1)$-Kronecker-Sum $G= \sum_{i=1}^{r+1} u_i\otimes A_i$, and such that $G'$ has minimum variance. 

Theorem \ref{thm:opt} follows directly from the following lemma.
\begin{lemma}
	Assume that Algorithm \ref{alg:Opt} gives a minimum-variance rank-$r$ approximation $C'=L'R'^T$ of the matrix $C=L^T R$ as constructed by Algorithm \ref{alg:OK}. Then Algorithm \ref{alg:OK} gives a minimum-variance unbiased approximator $G' = \sum_{i=1}^r{u'_i\otimes A'_i}$ of $G= \sum_{i=1}^{r+1} u_i\otimes A_i$.
\end{lemma}
\begin{proof}
	The first important observation is that for an optimal approximator $G'$, the random variables $u'_i$ are always elements of $\mathrm{span}\{u_1,\ldots,u_{r+1}\}$ and the $A'_i$ are always elements of $\mathrm{span}\{A_1,\ldots,A_{r+1}\}$. Otherwise, we could simply take the $u'_i$ and project them (orthogonally) onto $\text{span}\{u_1,\ldots,u_{r+1}\}$ (and similarly for the $A'_i$) to obtain a new unbiased approximator $G''$ of $G$ which has less variance than $G'$.
	
	From this observation, it follows that the $u'_i$ are a (random) linear combination of the $u_i$ (and similarly for $A'_i$). In order to be able to get simple closed-form expressions for the variance of $G'$, we now choose orthonormal bases of the spaces $\mathrm{span}\{u_1,\ldots,u_{r+1}\}$ and  $\mathrm{span}\{A_1,\ldots,A_{r+1}\}$, let us denote them by $v_1,\ldots,v_{r+1}$ and $B_1,\ldots,B_{r+1}$ respectively. Now define matrices $L,R\in\mathbb{R}^{(r+1)\times(r+1)}$ by setting $L_{i,j}:=\langle v_i,u_j\rangle$ and $R_{i,j}:=\langle B_i,A_j\rangle$. Especially, we have $u_{j} = \sum_{i=1}^{r+1} L_{i,j}v_i$ (and an analogous equation for $A_j$). Observe that the matrix $C:=LR^T$ has coefficients representing $G$ in terms of an orthogonal bases, more precisely 
	\begin{eqnarray}\label{G}
	G = \sum_{1\leq i,j\leq r+1} C_{i,j} \left(v_i \otimes B_j\right), 
	\end{eqnarray}
	where it is not difficult to see that the $\left(v_i\otimes B_j\right)_{i,j}$ are orthonormal.
	
	As noted above, the $u'_i,A'_i$ forming $G'$ are linear combinations of  $u_i, A_i$ respectively,  so they can be written in terms of the ONBs $u'_i,A'_i$. Let us record all the corresponding coefficients in (random) matrices $L',R'\in\mathbb{R}^{(r+1)\times r}$ meaning that $L'_{i,j}=\langle v_i, u'_j\rangle$ (or equivalently $u'_j = \sum_{i} L'_{i,j} v_i$ ) and $R'_{i,j} =\langle B_i,A_j\rangle$. The same calculations as above then show that, for the matrix $C':= L' R'^T$, we have 
	\begin{eqnarray}\label{G'}
	G' = \sum_{1\leq i,j\leq r+1} C'_{i,j} \left(v_i \otimes B_j\right). 
	\end{eqnarray}
	Now, with linearity of expectation and independence of $(u'_i\otimes A'_j)_{i,j}$, we can conclude from \eqref{G}, \eqref{G'} that 
	\[
	\E[G'] = G \quad\Leftrightarrow\quad \E[C'] = C. 
	\]
	From the orthonormality of $(u'_i\otimes A'_j)_{i,j}$ we moreover obtain
	\[
	\Var[G'] = \E\left[\sum_{i,j}\left(C'_{i,j}-\E[C'_{i,j}]\right)^2\right] = \Var[C'].
	\]
	Combined, the last two statements show that finding a minimum variance, unbiased approximator $G'$ of $G$ is equivalent to finding a minimum variance, unbiased rank-$r$ approximator $C' = L' R'^T$ of $C$. 
	This finishes the proof of the lemma. 
\end{proof}

\subsection{Proof of Theorem \ref{thm:mat}, Correctness of Algorithm \ref{alg:Opt}}\label{sec:mat}
Here, we prove Theorem \ref{thm:mat}. The calculations carried out in the proof precisely match the ones carried out by Algorithm \ref{alg:Opt}, so that its correctness is an immediate consequence of the proof given below. Theorem \ref{thm:mat} is a direct consequence of Theorems \ref{thm:optcond} and \ref{thm:achievable} below.

To simplify notation, we shall assume that $C$ has dimension $C\in\R^{n\times n}$, the more general case $C\in\R^{m\times n}$ can be proved in the same way without additional complications. The outline of the proof is as follows: We will first reduce finding an unbiased rank-$r$ approximator of $C$ to the problem of finding an unbiased, rank-$r$ approximator of a diagonal matrix $D$ using SVD. We will then use a duality argument to give a sufficient condition for an approximator of $D$ to have minimal variance and conclude by constructing an approximator fulfilling this condition.

\subsubsection{Reducing the problem to diagonal matrices}
In this subsection, we give the simple explanation of how finding an optimal rank-$r$ approximator $C'$ of $C\in\mathbb{R}^{n\times n}$ can be reduced to finding an optimal rank-$r$ approximator $D'$ of a diagonal matrix $D=\diag(d_1,\ldots,d_n)$ with non-negative entries. 
\begin{lemma}
	Let $C$ be as above and let $UDV^T = C$ be a SVD of $C$. Then, given an unbiased rank-$r$ approximator $D'$ of $D$, it holds that $C'=UD'V^T$ is an unbiased approximator of $C$. Moreover, $C'$ is optimal if  $D'$ is optimal.
\end{lemma}
\begin{proof}
	The proof is almost immediate. The fact that $C'$ unbiasedly approximates $C$ follows from the fact that $D'$ unbiasedly approximates $D$, linearity of expectation and $C=UDV^T$. Note that given $C'$, we can write $D' = U^T C' V$, so that there is a one-to-one correspondence between $C'$ and $D'$. Since $U,V$ are orthogonal, it follows that $\Var[C']=\Var[D']$, so that $C'$ is optimal if and only if $D'$ is optimal.
\end{proof}

\subsubsection{Optimally approximating diagonal matrices}
In this subsection, we construct a minimum-variance, unbiased approximator for diagonal matrices $D=\diag(d_1,\ldots,d_n)$ with non-negative, non-increasing entries. The first step is giving a sufficient condition for any such approximator to be optimal. The second step is the construction of an unbiased approximator satisfying this condition.

\subsubsection*{Sufficient Optimality Condition}
For stating our condition, we first need some notation. Let $D=\diag(d_1,\ldots,d_n)$ be a diagonal matrix such that $d_1\geq d_2\ldots \geq d_n\geq 0$.
Let 
$$m=\min\left\{i\colon (r-i+1)d_i\leq \sum_{j=i}^n d_j\right\},\quad k = r-m+1.$$ We can already give some intuition on the meaning of $m$. We will see later that it is defined so that the first $m-1$ diagonal entries are so large, that an optimal approximation consists of approximating $D$ by 
\begin{eqnarray}\label{eq:keep}
D'=
\left(
\begin{matrix}
\begin{matrix} d_1 &\ldots& 0 \\ 0 &\ddots& 0 \\ 0&\ldots &d_{m-1} \end{matrix}  & \large{0}  \\
\large{0}      & \large{D'_2}  \\
\end{matrix}
\right)
\end{eqnarray}
where $D'_2$ is an optimal, unbiased rank-$k$ approximator of $\diag(d_m,\ldots,d_n)$ (note that if the rank of $D'_2$ was larger then $k$ then the rank of $D'$ would be larger than $r$). In other words, some large diagonal entries are kept deterministically and only smaller ones are `mixed' into a matrix of lower rank.

Defining $$s_1:= \sum_{j=m}^n d_j\quad\text{and }\quad s_2:= \sum_{j=m}^n d_j^2$$ we can state our optimality condition.

\begin{thm}\label{thm:optcond}
	Let $D,m,k, s_1,s_2$ be as above. Then, any unbiased rank-$r$ approximator $D'$ of $D$ satisfies $$\Var[D']\geq \frac{s_1^2}{k} - s_2.$$
	Equality is achieved if and only if, in addition to being unbiased, $D'$ satisfies the following two conditions: 
	\begin{enumerate}
		\item\label{cond1} $D'$ is of the form given in equation \eqref{eq:keep}, such that  
		\item\label{cond:close} $D'_2$ always (with probability 1) satisfies $${\left\lVert \frac{k}{s_1}D_2' -  \Id_{n-(m-1)}\right\rVert^2 = n-r}.$$
	\end{enumerate}
\end{thm}
Before proving the theorem, let us explain Condition \ref{cond:close}.~of the theorem. Note that $D'_2$ has (square) dimension $n-(m-1)$ and must have rank at most $k=r-(m-1)$. Thus, by the Eckart-Young Theorem
\begin{eqnarray*}\left\lVert \frac{k}{s_1}D_2' -  \Id\right\rVert^2 \geq \bigl((n-(m-1)) - k))\bigr)
	=n-r.
\end{eqnarray*}
In other words, the approximator $D'_2$ is optimal, if and only if $\frac{k}{s_1}D'_2$ is as close to $\Id$ as it can be (given its rank).
\begin{proof}[Proof of Theorem \ref{thm:optcond}]
	As mentioned before, we will prove the theorem using a duality argument. Let $D'$ be an unbiased rank-$r$ approximator of $D$. Observe that for any matrix $B \in \R^{n\times n}$, due to linearity of expectation, it holds that $\E\bigl[\Tr[(D'-D)B]\bigr] = 0$. We can therefore write 
	\begin{align}
	\Var[D'] &= \E\Biggl[\Tr\left[(D'-D)(D'-D)^T\right]\Biggr] \nonumber\\
	&= \E\Biggl[\Tr\left[(D'-D)(D'-D)^T\right] +2\Tr\left[(D'-D)B^T\right] \Biggr] \nonumber\\
	&= \E\Biggl[\Tr\left[\left(D'-D +B\right)\left(D'-D + B\right)^T\right] -\Tr\left[BB^T\right]\Biggr] \nonumber\\
	&= \E\left[\lVert D'- \left(D-B\right)\rVert^2\right]  - \E\left[\lVert B\rVert^2\right] \nonumber\\
	\label{eq:lowerBound}
	& \geq \min_{\substack{X\in \R^{n\times n}, \\ \text{rank}(X)\leq r} }\Bigl(\lVert X - (D-B)\rVert^2\Bigr) - \lVert B\rVert^2. 
	\end{align}
	Thus, for any $B \in \R^{n\times n}$, we get the lower bound from equation \eqref{eq:lowerBound} on the variance of $D'$. 
	We now choose $B$ to maximize the lower bound. Namely, we choose $B$ so that 
	\[
	D-B = 
	\left(
	\begin{matrix}
	\begin{matrix} d_1 &\ldots& 0 \\ 0 &\ddots& 0 \\ 0&\ldots &d_{m-1} \end{matrix}  & \large{0}  \\
	\large{0}      & \large{\frac{s_1}{k}\Id_{r-(m-1)}}  \\
	\end{matrix}
	\right).
	\]
	This implies 
	\begin{eqnarray}
	\lVert B\rVert^2 &=& \sum_{j=m}^n \left(d_j - \frac{s_1}{k}\right)^2 \nonumber\\
	&=& s_2 - 2\frac{s_1^2}{k} + (n-m+1)\frac{s_1^2}{k^2}.\label{eq:B}
	\end{eqnarray} 
	Moreover, note that the diagonal entries of $D-B$ are non-increasing. For $m=1$ this is immediate, for $m>1$, the definition of $m$ implies $ (r-(m-1) + 1) d_{m-1} > \sum_{j= m-1}^n d_j$ giving $d_{m-1}>\frac{s_1}{k}$ showing that diagonal entries are indeed non-decreasing. Thus, by the Eckart-Young Theorem, we have
	\begin{eqnarray}\label{eq:min}
	\min_{\substack{X\in \R^{n\times n}, \\ \text{rank}(X)\leq r} }\Bigl(\lVert X - (D-B)\rVert^2\Bigr) \geq (n-r)\frac{s_1^2}{k^2}.
	\end{eqnarray}
	We now obtain the statement of the theorem by plugging \eqref{eq:B} and \eqref{eq:min} into \eqref{eq:lowerBound} and recalling $k = r-m+1$
	\begin{eqnarray*}
		\Var[D'] &\geq &\frac{s_1^2}{k^2}\biggl((n-r) + 2k - (n-m+1) \biggr) -s_2 \\
		&=& \frac{s_1^2}{k} - s_2.
	\end{eqnarray*}
Note that equality is achieved if and only if it is always achieved in \eqref{eq:min}. Since $d_{m-1}>\frac{s_1}{k}$, it follows from the Eckart-Young Theorem that equality  in \eqref{eq:min} is achieved if an only if the Conditions \ref{cond1} and \ref{cond:close} from the theorem hold. 
\end{proof}

\subsubsection*{Construction of Approximator fulfilling the Optimality Condition}
We now show that the condition from Theorem \ref{thm:optcond} can be satisfied by a rank-$r$ approximator $D'$ of $D$.  
\begin{thm}\label{thm:achievable}
	In the setting of Theorem \ref{thm:optcond}, there is an unbiased approximator $D'$ of $D$ satisfying the optimality Conditions \ref{cond1} and \ref{cond:close} from Theorem \ref{thm:optcond}.
\end{thm}
To simplify the exposition of the proof of this theorem, we state two lemmas. Their proofs are given in the end of this section. 
\begin{lemma}\label{lem:rescaled}
	Let $D= \diag(d_1,\ldots,d_n)$ such that $d_1,\ldots,d_n\in [0,1]$ with $\sum_{i=1}^n d_i = r$ a positive integer. Moreover, assume there exist orthonormal vectors $z_1,\ldots, z_r\in\R^{n\times 1}$ so that the matrix $Z=\sum_{i=1}^r z_i z_i^T$ has diagonal entries $d_1,\ldots,d_n$ (in this order). For a vector $s\in\R^{n\times1}$ of uniformly random signs (i.e. each entry is chosen uniformly and independently from $\{\pm1\}$), and $z'_i = s\odot z_i$, define $D' = \sum_{i=1}^k z'_i {z'_i}^T$.  
	
	Then $D'$ is an unbiased rank-$r$ approximator of $D$ satisfying the optimality Conditions \ref{cond1} and \ref{cond:close} from Theorem \ref{thm:optcond}.
\end{lemma}
The pointwise multiplication by random signs $s$ in this lemma can be interpreted as a generalization of the `sign-trick' from~\cite{tallec2017unbiased}.

To make use of the above lemma, we need to construct $z_1,\ldots, z_r$ as described in its statement. This is achieved by the following lemma, whose proof uses ideas from~\cite{post}
\begin{lemma}\label{lem:diag}
	Let $D=\diag(d_1,\ldots,d_n)$ such that $d_1,\ldots,d_n\in[0,1]$ with $\sum_{i=1}^n d_i = r$ a positive integer. Then, there exists orthonormal vectors $z_1,\ldots,z_r\in\R^{n\times1}$ so that the matrix $Z:= \sum_{i=1}^r z_i z_i^T$ has the same diagonal entries $d_1,\ldots,d_n$ as $D$ (in this order). 
\end{lemma}
Note that $Z$ as defined in this lemma is a symmetric idempotent matrix with trace $r$. It is not difficult to show that every symmetric, idempotent matrix $Z$ with trace $r$ can be decomposed as a sum $Z=\sum_{i=1}^r z_iz_i^T$, where the $z_i$ are orthonormal. So the Lemma can also be interpreted as the following statement about symmetric, idempotent matrices: Symmetric idempotent matrices can have any diagonal up to the constraint that diagonal entries are between 0 and 1 and sum up to an integer. It is easy to check that any symmetric, idempotent matrix also satisfies these two conditions, so that the lemma fully classifies the diagonals of symmetric, idempotent matrices.

We are now ready to give the proof of Theorem \ref{thm:achievable}.

\begin{proof}[Proof of Theorem \ref{thm:achievable}]
	Note that in order to construct an optimal rank-$r$ approximator $D'$ of $D$ it suffices to find a rank-$k$ approximator $D'_2$ of $D_2 = \diag(d_m,\ldots,d_n)$ satisfying condition \ref{cond:close}. from Theorem \ref{thm:optcond}. 
	
	Note that $\frac{k}{s_1}D_2$ satisfies the conditions of Lemma \ref{lem:diag}, since its diagonal entries sum to $k$ and are in $[0,1]$ by the definition of $m$. Therefore, there exist orthonormal vectors $z_1,\ldots, z_k\in\R^{(n-m+1)\times 1}$ so that $Z=\sum_{i=1}^k z_iz_i^T$ has the same diagonal as $\frac{k}{s_1}D_2$. By Lemma \ref{lem:rescaled}, choosing a vector $s\in\R^{(n-m+1)\times1}$ of random signs (i.e. each entry is uniformly and independently drawn from $\{\pm1\}$) gives an optimal unbiased rank-$k$ approximator $\sum_{i=1}^k (s\odot z_i)(s\odot z_i)^T$ of $\frac{k}{s_1}D_2$ satisfying the (rescaled) Conditions \ref{cond1} and \ref{cond:close} from Theorem \ref{thm:optcond}. Multiplying this approximator by $\frac{s_1}{k}$ therefore gives an unbiased rank-$k$ approximator of $D_2$ satisfying the same conditions.
\end{proof}

\subsubsection*{Proof of Lemma \ref{lem:rescaled}}
\begin{proof}
	We will first check that $D'$ is actually an unbiased approximator of $D$ and then check the condition from Theorem \ref{thm:optcond}. 
	
	In order to show $\E[D']=D$, consider the $(a,b)$-th entry $(z'_i {z'_i}^T)_{a,b}$ of the matrix $z'_i{z'_i}^T$. Observe that $(z'_i {z'_i}^T)_{a,b}= s_as_b (z_i z_i^T)_{a,b}$, so that for $a\neq b$ we have $\E[(z'_i {z'_i}^T)_{a,b}]= 0$ and for $a=b$ we have $\E[(z'_i {z'_i}^T)_{a,a}]= (z_i z_i^T)_{a,a}$. From this, it follows that $\E[D']$ has 0 off-diagonal entries and that its diagonal entries equal the ones of $Z$. In other words, $\E[D']=D$ as desired.
	
	We now check the conditions given by Theorem \ref{thm:optcond}. First of all, note that in the notation of the theorem we have $m=1$ (since $d_1\leq 1$ by assumption) and therefore $s_1 =r$, so that we just have to show $\Vert D' - \Id_r\Vert^2 = (n-r)$. This is immediate from the fact that the $z'_i$ always inherit orthonormality from the $z_i$ and Observation \ref{obs}.
\end{proof}

\subsubsection*{Proof of Lemma \ref{lem:diag}}
Let us first give a simplified construction for the special case $n=r+1$ which is used by Algorithm \ref{alg:Opt}.
We simply define the unit-norm vector 
${z_0 = (\sqrt{1-d_1},\ldots,\sqrt{1-d_{r+1}})^T}$. 
Now, we find $z_1,\ldots,z_{r}$ completing an orthonormal bases $z_0,z_1,\ldots,z_r$ of $\R^{r+1}$(for example, one can first complete the basis arbitrarily and then apply the (modified) Gram-Schmidt algorithm). We then have $\sum_{i=0}^{r} z_i z_i^T = \Id_n$ and therefore $Z=\sum_{i=1}^{r} z_i z_i^T= \Id_n - z_0z_0^T$ has the desired diagonal entries.

We now give the full proof of the lemma, it uses ideas from~\cite{post}.
\begin{proof}
	First, we may without loss of generality assume that $d_1\geq\ldots \geq d_n$, since reordering the diagonal entries of $Z$ can be achieved by reordering the coordinates of the $z_i$.
	
	We will prove the statement by induction on $r$ and we note that the proof can easily be turned into an algorithm constructing the $z_i$. 
	
	For $r=1$, the statement is trivial: Simply set $z_1 = (\sqrt{d_1},\ldots,\sqrt{d_n})$ and note that it has norm 1.
	
	Now assume the statement holds for $r-1$. We want to show that it holds for $r$. Our plan is as follows: We will change two diagonal entries $d_m, d_{m+1}$, so that the first $m$ diagonal entries sum up to $1$ and the remaining ones sum up to $r-1$. We then apply the induction hypothesis to find vectors $x_i$ such that $\sum_{i}x_ix_i^T$ has the slightly changed values on the diagonal (with new $d_m,d_{m+1}$) and then apply a rotation $R$ to restore the original diagonal entries $d_m,d_{m+1}$ and giving the desired $z_i=Rx_i$.
	
	 We now give the details. Set $$m = \max\left\{j\in\{1,\ldots,n\}\colon \sum_{t=1}^j d_t \leq 1     \right\}$$ and let $\alpha = 1 - \sum_{t=1}^j d_t$. Now, set $d'_i = d_i$ for $i\neq m,m+1$ and set $d'_{m}= d_{m} + \alpha$ as well as $d'_{m+1}= d_{m+1} -\alpha$ (note that $1\leq m < r$ so that $m, m+1$ are valid indices). Then we claim that $d'_{m+1},\ldots,d'_{n}$ satisfy the conditions of the lemma for $r-1$.
	
	This is not difficult to see: Note that $\sum_{i=1}^m d'_i = \sum_{i=1}^m d_i +\alpha = 1$ by the definition of $\alpha$. Moreover, $\sum_{i=1}^n d'_i = \sum_{i=1}^n d_i = r.$ Therefore, we get $\sum_{i=m+1}^{n} d'_i = r -\sum_{i=1}^m d'_i = r-1$. Moreover, $d'_{m+1}\leq d_{m+1} \leq 1$ and 
	\begin{eqnarray*}d'_{m+1} &=& d_{m+1} - \alpha = d_{m+1} - \left(1-\sum_{i=1}^m d_i\right) \\
		&=& \sum_{i=1}^{m+1} d_i -1\geq 0
	\end{eqnarray*}
	by definition of $m$. For $i>m+1$, the condition $d'_i\in[0,1]$ is trivial. So we have indeed checked that $d'_{m+1},\ldots,d'_{n}$ satisfy the conditions of the lemma for $r-1$.
	
	By induction, there exist vectors $y_1,\ldots,y_{r-1}\in\R^{(n-m)\times 1}$, so that $Y = \sum_{i=1}^{r-1} y_i{y_i}^T$ has diagonal entries $d'_{m+1},\ldots,d'_{n}$. We write $q= (\sqrt{d'_1},\ldots, \sqrt{d'_m})^T$ and let $${x_i = \left(\begin{matrix} 0\\y_i \end{matrix}\right) \in\R^{n \times 1}}
	\;\text{and }\;
	{x_r = \left(\begin{matrix} q\\0_{}\end{matrix}\right)}\in\R^{n\times 1}.$$ Then, the $x_i$ are clearly orthonormal and the matrix $X = \sum_{i=1}^{r} x_i x_i^T$ can be written as a block diagonal matrix of the form
	\[
	X =
	\left(
	\begin{matrix}
	qq^T & 0 \\
	0 &  Y
	\end{matrix}
	\right).
	\]
	Especially, $X$ has diagonal entries $d'_1,\ldots, d'_n$. On top of that, when we restrict the indices of $X$ to be $m$ or $m+1$, we obtain the submatrix $$\left(\begin{matrix}d'_{m} & 0 \\ 0& d'_{m+1}\end{matrix}\right)=\diag(d_{m}+x,d_{m+1}-x)=:D_m.$$ 
	
	Let $$R(\phi) = \left(\begin{matrix}\cos \phi & \sin\phi \\ -\sin\phi& \cos\phi\end{matrix}\right)$$ be a rotation matrix (with angle $\phi$) and choose $\phi$ so that $R(\phi)D_mR(\phi)^T$ has diagonal entries $d_{m}, d_{m+1}$, i.e. $\phi = \arcsin\left(\sqrt{\frac{x}{2x + d_m- d_{m+1}}}\right)$, which is well-defined since $d_m\geq d_{m+1}$. 
	
	Now, consider the block-diagonal matrix 
	$$R=\left(
	\begin{matrix}
	\Id_{m-1} & 0 & 0 \\
	0 & R(\phi) & 0 \\
	0 & 0 & \Id_{n-m-1}
	\end{matrix}
	\right).
	$$
	By the choice of $\phi$, we then get that $RXR^T = \sum_{i=1}^r (Rx_i)(Rx_i)^T$ has diagonal entries $d_1,\ldots,d_n$. Since $R$ is orthogonal, we get that the $z_i = Rx_i$ are orthonormal and we have therefore constructed the $z_i$ as desired.
\end{proof}

\newpage
\section{Additional Experiments}\label{sec:exp}

	Here we include $5$ additional experiments complementing the ones presented in the main paper. The first one illustrates that the batch size chosen does not affect the observation that the performance of OK matches that of TBPTT. The next three analyze the cosine between the true gradient and the approximated one. The first of the three shows the cosine for untrained networks on CHAR-PTB. The second of the three shows the cosine on the Copy task after training until the algorithm learns sequences of length $40$. The third one shows that the specific set of trained weights does not affect the cosine significantly, by repeating the previous experiment while retraining the network. The last experiment analyzes the quality of $r$-OK, the optimal unbiased Kronecker rank $r$ approximation, from a different point of view: by comparing it to the 'best' biased rank $r$ approximation of the gradient, that is, the approximation that stores the closest approximation of the gradient as an $r$-Kronecker-Sum. Intuitively, the performance of the biased version of the algorithm measures how far away the gradient is from a low rank approximation, which also influences how well one can do unbiased low-rank approximations of the gradient.

	\subsection{CHAR-PTB with larger batch size}
		For the first experiment, we would like to illustrate that the results obtained in Figure~$2$, regarding $8$-OK matching TBPTT-$25$ did not depend on the batch size. As noted in the paper, this is in principle clear as the batch size $b$ divides the batch noise and the approximation noise by the same factor $b$. Figure~\ref{fig:ptb_comparison_b64} shows the validation performance of $8$-OK and TBPTT-$5$ and $25$ on the Penn TreeBank dataset in bits per character (BPC). The experimental setup is exactly the same as in Figure~$2$, except the batch size chosen is $64$. Table~\ref{table:ptb_val_test_b64} summarizes the results.

		\begin{figure}[ht]
			\vskip 0.2in
			\begin{center}
				\centerline{\includegraphics[width=\columnwidth]{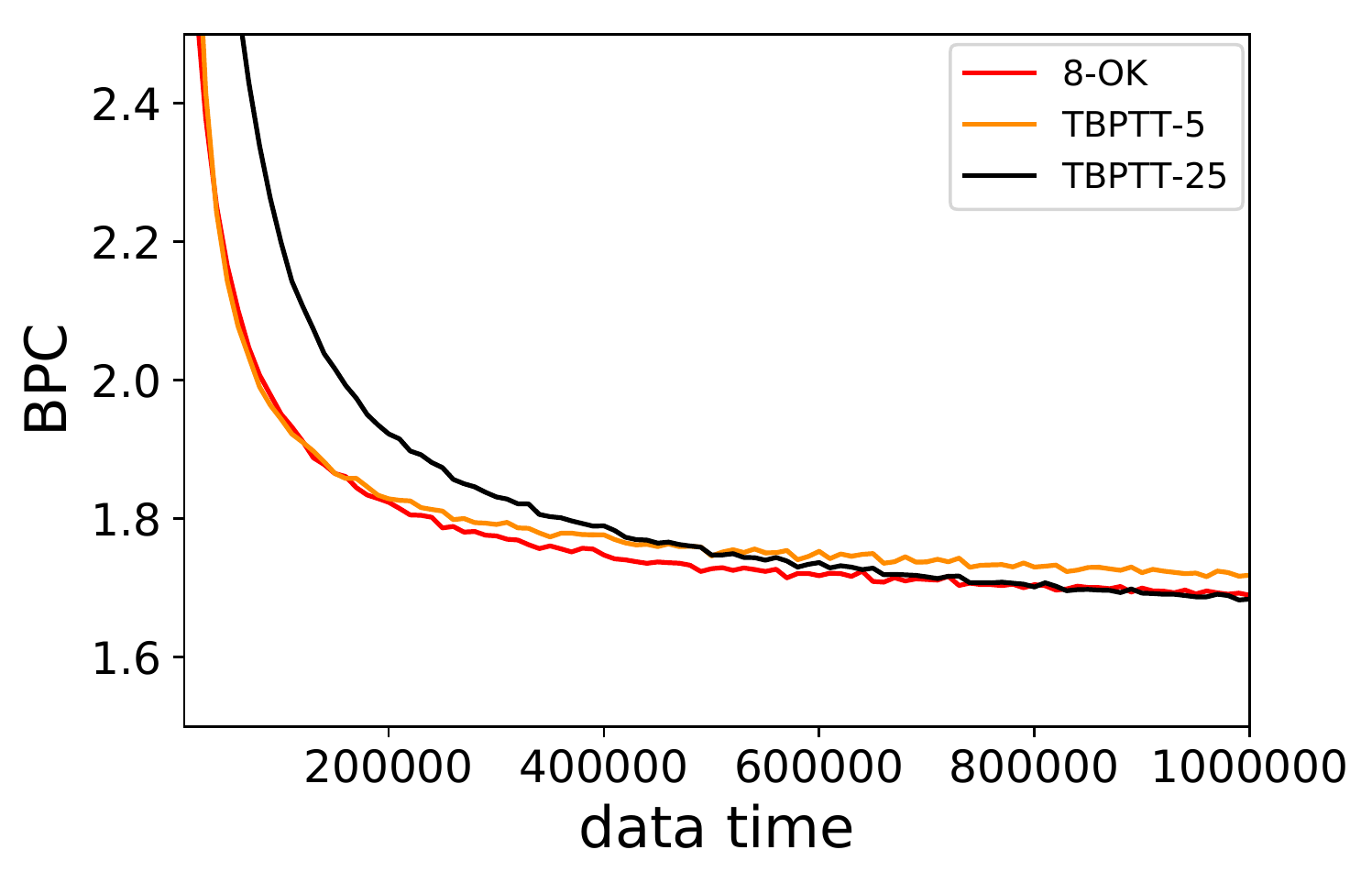}}
				\caption{Validation performance on CHAR-PTB in bits per character (BPC). Even with larger batch sizes, $8$-OK matches the performance of TBPTT-$25$. We trained a RHN with $256$ units, with a batch size of $64$.}
				\label{fig:ptb_comparison_b64}
			\end{center}
			\vskip -0.2in
		\end{figure}

	\begin{table}[h]
		\caption{Results on Penn TreeBank with a batch size $64$. Standard deviations are smaller than $0.01$.}
		\label{table:ptb_val_test_b64}
		\vskip 0.15in
		\begin{center}
			\begin{small}
				\begin{sc}
					\begin{tabular}{lcccr}
						\toprule
						Name & Validation & Test & \#params \\
						\midrule
						$8$-OK    & $1.69$ & $1.64$ & $133K$ \\
						TBPTT-$5$    & $1.72$ & $1.67$ & $133K$ \\
						TBPTT-$25$    & $1.68$ & $1.63$ & $133K$ \\
						\bottomrule
					\end{tabular}
				\end{sc}
			\end{small}
		\end{center}
		\vskip -0.1in
	\end{table}

	\subsection{Cosine analysis between the true gradient and the approximated one}
	
	For the second experiment, we pick an untrained RHN with $256$ units in the CHAR-PTB task. This contrasts with Figure~$3$, where we first trained the network weights to assess the gradient estimate at the end of training (which, as indicated there, is more challenging). We compute the cosine of the angle $\phi$ between the gradient estimates provided by OK and KF-RTRL and the true RTRL gradient for $10000$ steps. We plot the mean and standard deviation for $20$ different untrained RHNs with random weights. Figure~\ref{fig:noise_untrained_ptb_comparison} shows that the gradient can be almost perfectly approximated by a sum of $2$ Kronecker factors, at least at the start of training. This illustrates the advantage of using an optimal approximation, as opposed to the one in KF-RTRL. 

	\begin{figure}[ht]
		\vskip 0.2in
		\begin{center}
			\centerline{\includegraphics[width=\columnwidth]{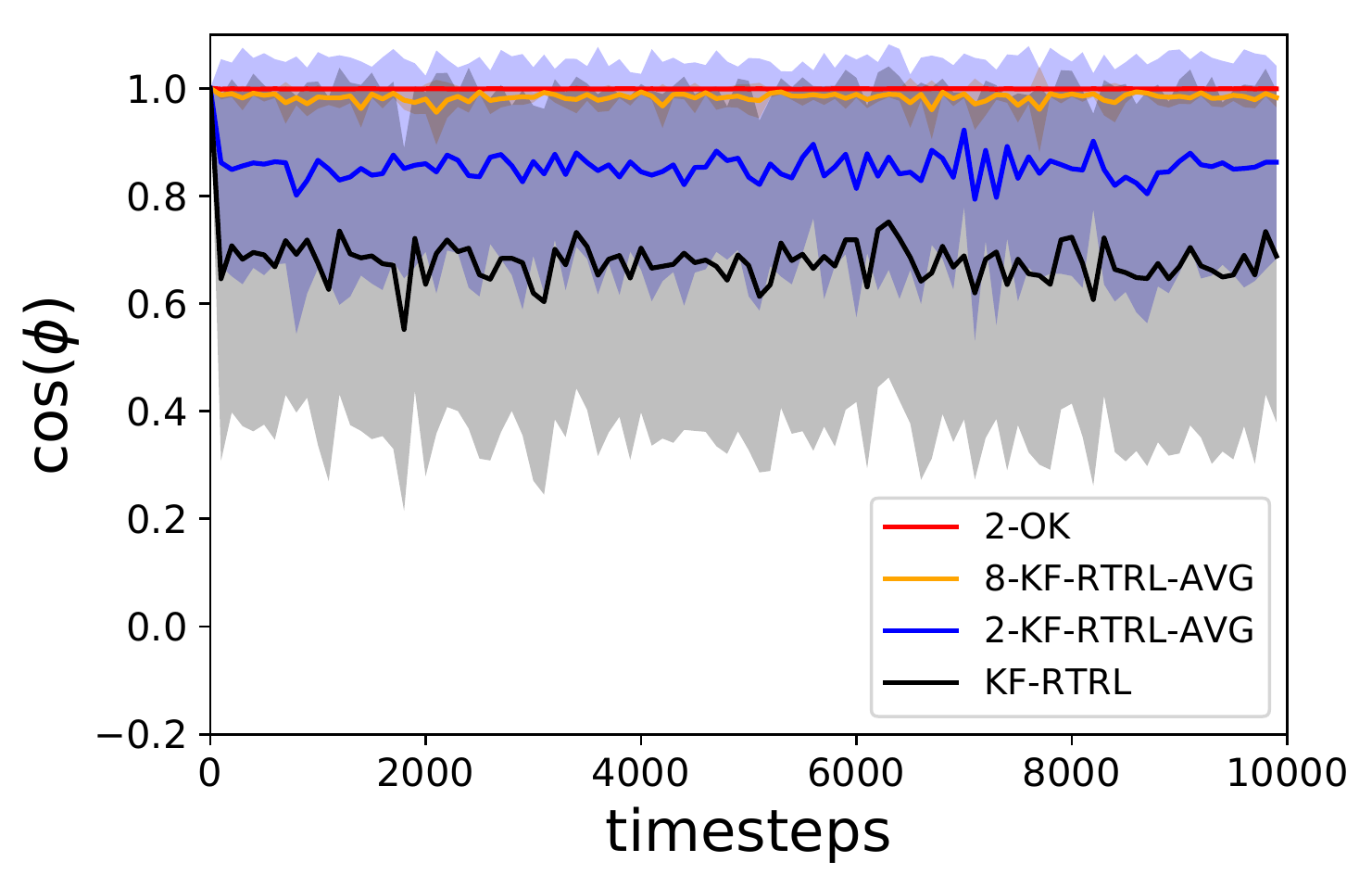}}
			\caption{Variance analysis on an untrained RHN for the CHAR-PTB task. At the start of training, even a sum of $2$ Kronecker factors suffices to perfectly capture the information in the gradient.}
			\label{fig:noise_untrained_ptb_comparison}
		\end{center}
		\vskip -0.2in
	\end{figure}

	Naturally, as shown in the paper, the most interesting behavior appears later in training. The third experiment in the appendix is equivalent to the one performed to produce Figure $3$, except we use the Copy task and a RHN with $128$ units trained until it learns a sequence of length $40$. The results are similar in spirit to the Figure shown in the main paper and are shown in Figure~\ref{fig:cp_noise_trained_exploration}. Observe that datapoints where the true gradient is smaller than 0.0001 were removed. This is necessary in the Copy task because there are a lot of steps (say when the network is reading the input), where the task is trivial and the performance has already saturated (leading to very small gradients). Of course, small gradient steps are also irrelevant for learning so removing them is justified.

	\begin{figure}[ht]
		\vskip 0.2in
		\begin{center}
			\centerline{\includegraphics[width=\columnwidth]{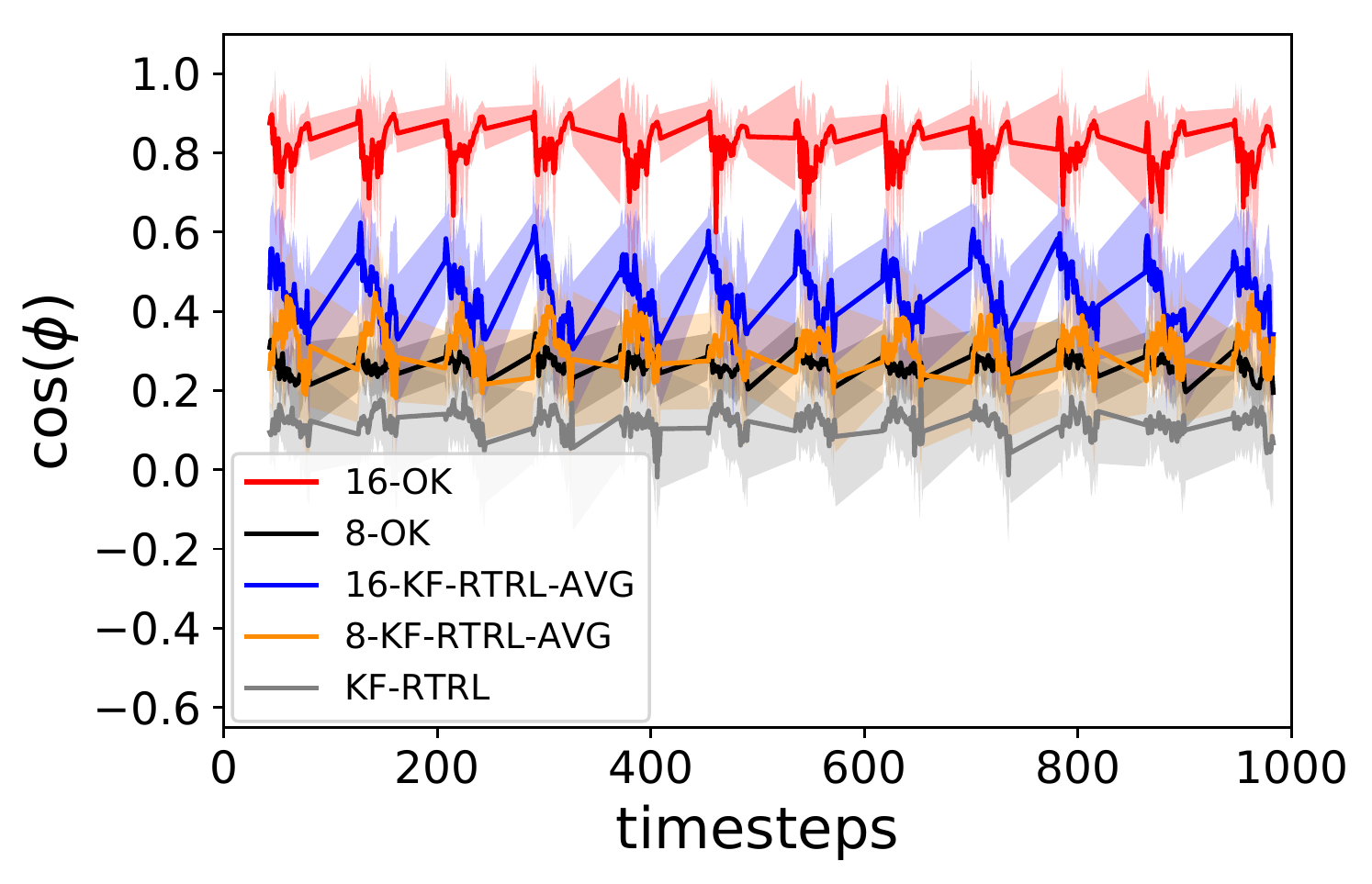}}
			\caption{Variance analysis on the Copy task for a RHN trained until it has learned sequences of length $T=40$. Even later in training, $16$-OK keeps a very good estimate of true gradient. For this plot, we remove datapoints where the true gradient is smaller than 0.0001, as those are irrelevant for learning. In particular, the steps corresponding to the network reading the input are not plotted.}
			\label{fig:cp_noise_trained_exploration}
		\end{center}
		\vskip -0.2in
	\end{figure}

	One might wonder whether the behavior observed in Figure~\ref{fig:cp_noise_trained_exploration} was not specific to the set of trained weights used there. To that end, we retrain the network and repeat the experiment. Figure~\ref{fig:noise_memo_2_comparison} shows that the behavior of the cosine does not depend much on the particular set of trained weights used.

	\begin{figure}[ht]
		\vskip 0.2in
		\begin{center}
			\centerline{\includegraphics[width=\columnwidth]{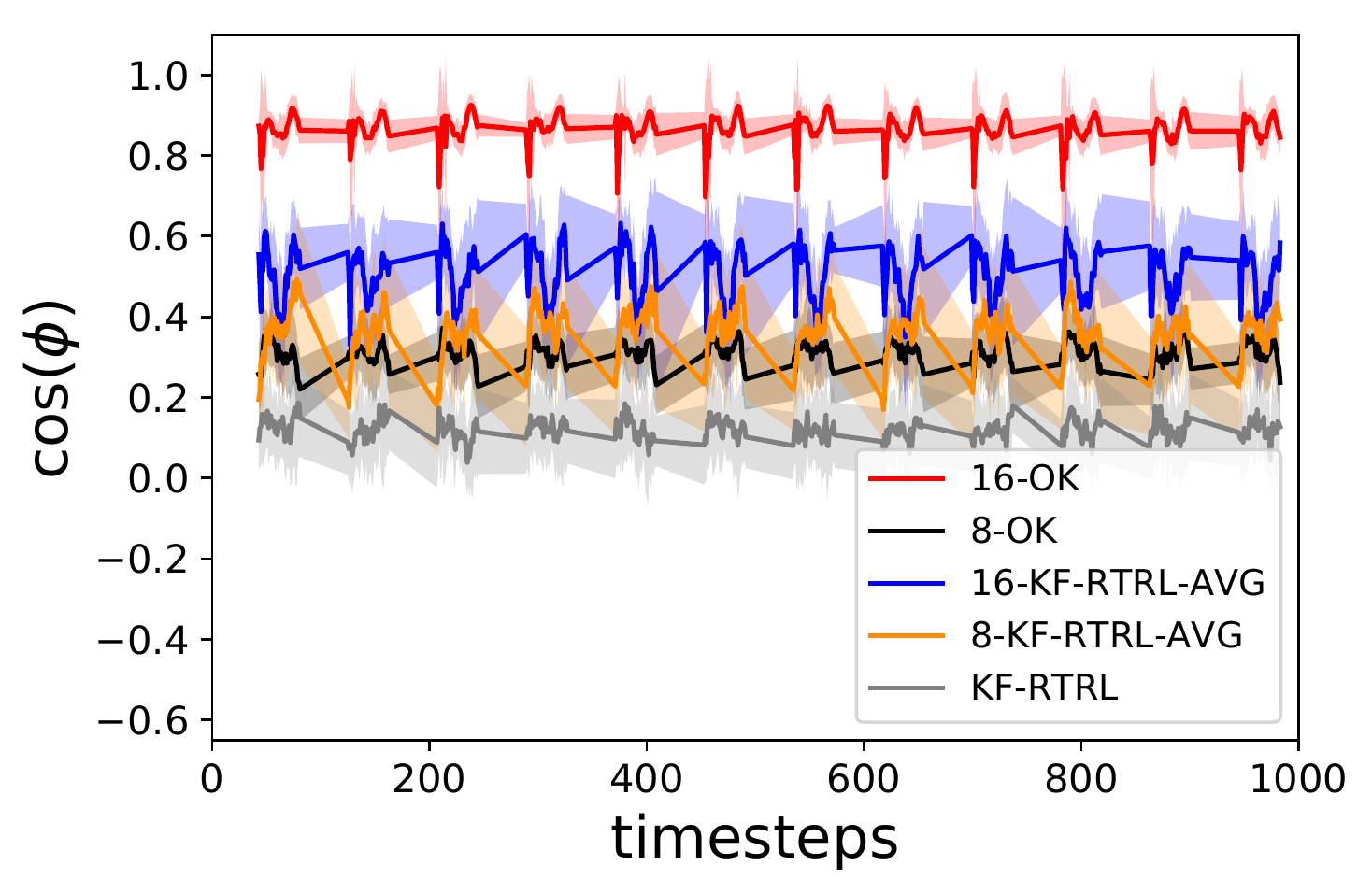}}
			\caption{Variance analysis on the Copy task for a RHN trained until it has learned sequences of length $T=40$. Repeated experiment with retrained weights.}
			\label{fig:noise_memo_2_comparison}
		\end{center}
		\vskip -0.2in
	\end{figure}

	Lastly, we analyze the effect of changing the number of units in the RHN. First, we pick untrained RHNs with sizes as powers of $2$ from $8$ to $512$ in the CHAR-PTB task. We compute the cosine of the angle $\phi$ between the gradient estimates provided by OK and KF-RTRL and the true RTRL gradient after $100$ steps. We plot the mean and standard deviation for $10$ different untrained RHNs with random weights (in the case of KF-RTRL and $2$-KF-RTRL-AVG, we use $100$ untrained RHNs). Figure~\ref{fig:noise_untrained_ptb_units} shows that the number of units does not affect the results seen in Figure~\ref{fig:noise_untrained_ptb_comparison}, at least for an untrained network.
	
	\begin{figure}[ht]
		\vskip 0.2in
		\begin{center}
			\centerline{\includegraphics[width=\columnwidth]{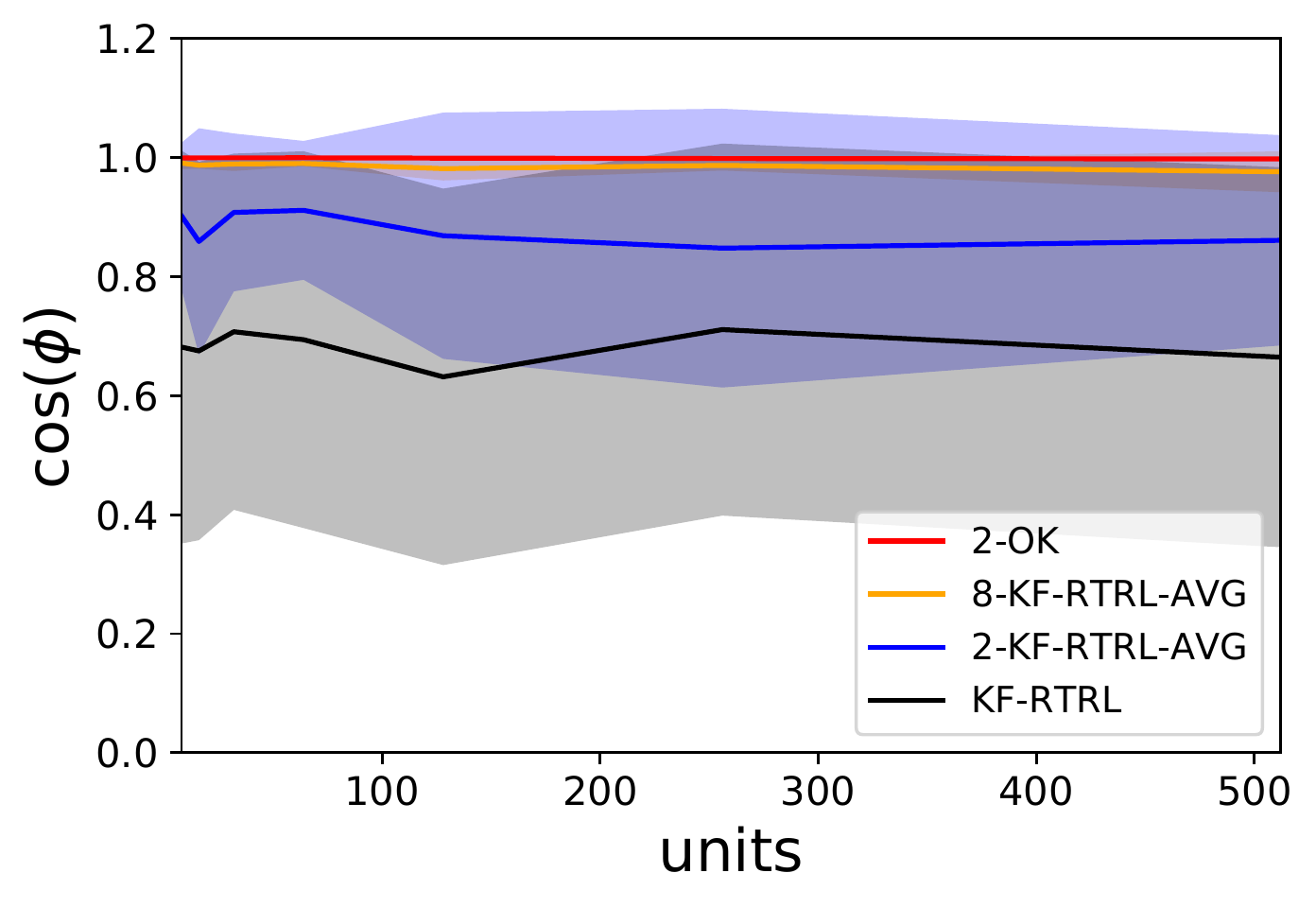}}
			\caption{Variance analysis on an untrained RHN for the CHAR-PTB task, varying the number of units from $8$ to $512$.}
			\label{fig:noise_untrained_ptb_units}
		\end{center}
		\vskip -0.2in
	\end{figure}
	
	As mentioned above, the most interesting behavior occurs at the end of training. To this end, we make Figures~\ref{fig:noise_trained_cp_n_512} and~\ref{fig:noise_trained_ptb_n_512} analogous to Figure~\ref{fig:cp_noise_trained_exploration} and Figure~$3$ from the main paper, where we include also an RHN of size $512$ for comparison. Observe that there is only a small difference in the performance of both OK and KF-RTRL when the network size is increased in Figure~\ref{fig:noise_trained_cp_n_512}. However, in Figure~\ref{fig:noise_trained_ptb_n_512}, OK drops more than KF-RTRL, with the advantage of using the optimal approximation almost completely vanishing. We believe that this is due to the gradients in the larger network containing longer term information than the gradients in the smaller network (that is, taking longer to vanish, due to the spectral norm of $H_t$ being closer to $1$). In particular, this effect is not present in Figure~\ref{fig:noise_trained_cp_n_512}, as both networks were trained until they learned sequences of length $40$. As a result, the gradients probably contain comparable amount of long term information. Naturally, the better test of the quality of the approximations used would be to train a network of larger size in either task. However, due to the computational costs, we have been unable to fully explore the effect of changing the network size experimentally.


	\begin{figure}[ht]
		\vskip 0.2in
		\begin{center}
			\centerline{\includegraphics[width=\columnwidth]{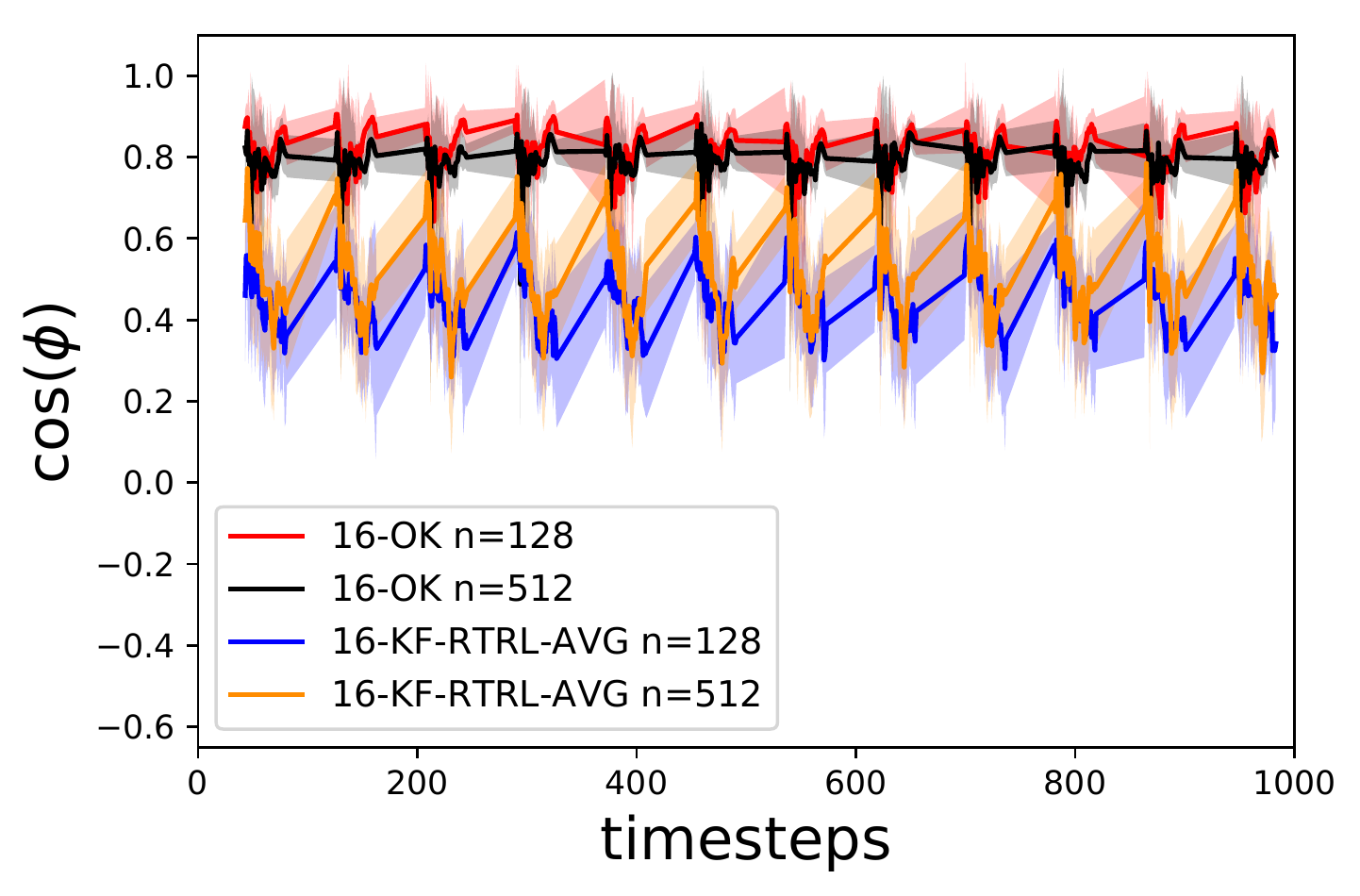}}
			\caption{Variance analysis on the Copy task for a RHN trained until it has learned sequences of length $T=40$. We vary the size of the RHN used, to show that both the OK and KF-RTRL-AVG approximations do not decay significantly, even later in training, for larger network sizes. As in Figure~\ref{fig:cp_noise_trained_exploration}, we remove datapoints where the true gradient is smaller than 0.0001.}
			\label{fig:noise_trained_cp_n_512}
		\end{center}
		\vskip -0.2in
	\end{figure}

	\begin{figure}[ht]
		\vskip 0.2in
		\begin{center}
			\centerline{\includegraphics[width=\columnwidth]{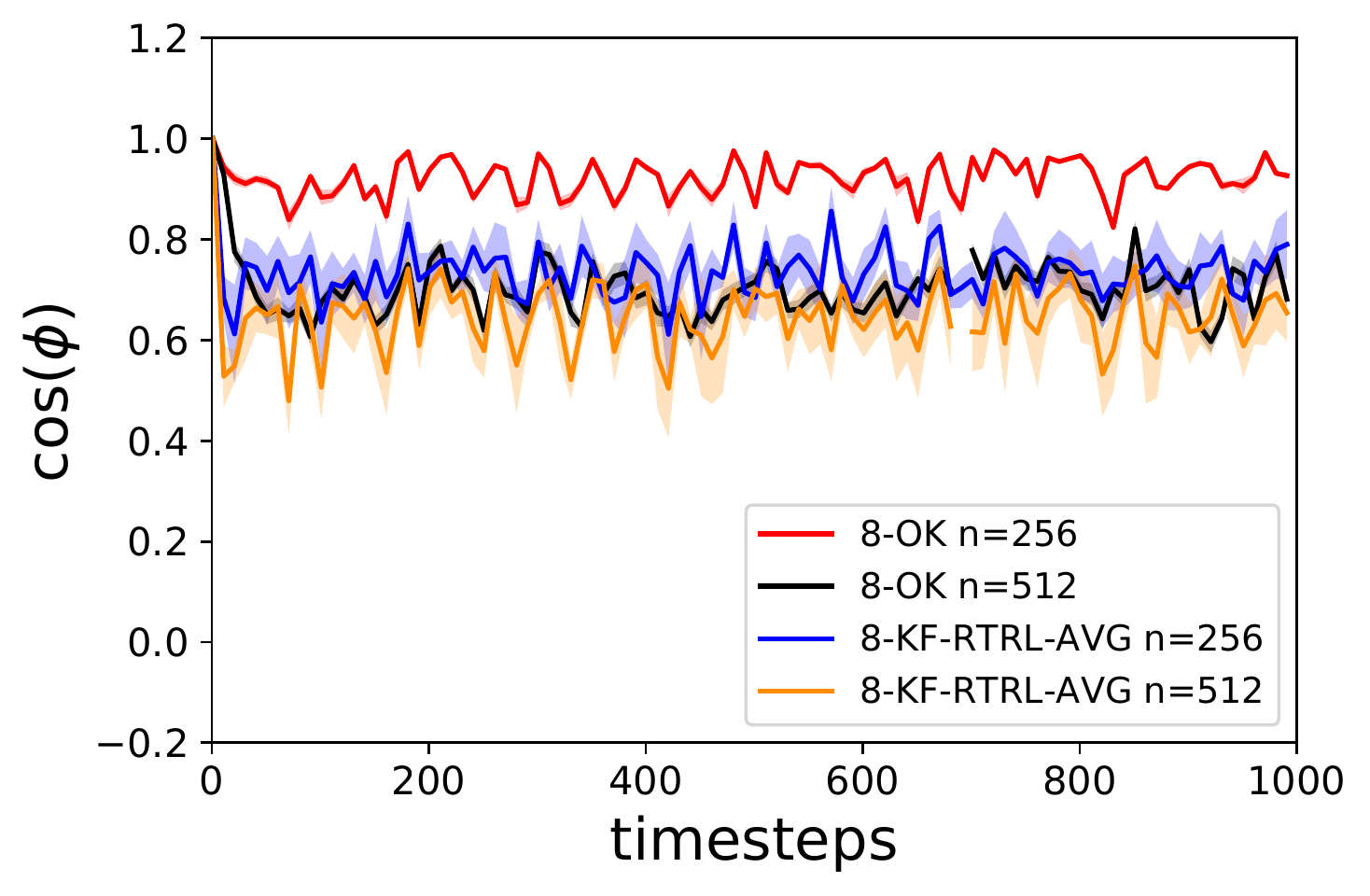}}
			\caption{Variance analysis on the Copy task for a RHN trained for $1$ million steps, with sizes either $256$ or $512$. Observe that $8$-OK decays more than $8$-KF-RTRL-AVG with the increase in network size. As in Figure~\ref{fig:cp_noise_trained_exploration}, we remove datapoints where the true gradient is smaller than 0.0001.}
			\label{fig:noise_trained_ptb_n_512}
		\end{center}
		\vskip -0.2in
	\end{figure}

	\subsection{Bias experiments}
	For the last set of experiments, we perform a Copy task experiment where we compare the optimal unbiased approximations used in OK to the corresponding optimal, biased ones. We first describe the biased approximations and then present the experiments.
	
	\subsubsection{Description of the Optimal Biased Approximation}
	In the paper, we were faced with approximating an $(r+1)$-Kronecker-Sum
	\begin{eqnarray}\label{eq:product}
	G = u_1\otimes (H_t A_1) +\ldots + u_r \otimes (H_t A_r) + h\otimes D
	\end{eqnarray}
	by an $r$-Kronecker-Sum $G'$. We solved the problem of finding an optimal, \textit{unbiased} approximator $G'$ of $G$. Instead, one can also construct an optimal biased approximator. Concretely, this means approximating $G$ by a (fixed, non-random) $r$-Kronecker-Sum $G'$, which minimizes $\lVert G - G'\rVert$. To clearly distinguish between unbiased and biased approximations, we refer to the corresponding algorithms as Unbiased Optimal Kronecker-Sum, $r$-U-OK, and Biased Optimal Kronecker-Sum, $r$-B-OK. 
	
	We now give details of how to construct $r$-B-OK. Similarly to $r$-U-OK, we first reduce the problem to approximating a matrix $C\in\R^{(r+1)\times(r+1)}$ optimally by a rank-r matrix $C'$ (which is now deterministic). The steps are exactly the ones given in Section \ref{sec:opt} and the matrix $C$ is also the same as the one presented there. Now, we need to minimize $\lVert C- C'\rVert$ subject to $C'$ having rank at most $r$. This is a well known problem and solved by the Eckart-Young Theorem (see Section \ref{sec:prelim}). This finishes the construction of $r$-B-OK. We also note that almost the same pseudo-code as Algorithm 2 from the paper (Algorithm \ref{alg:OK}) can be used. Rather than calling $Opt(C)$, we need to call $OptBias(C)$ as described in Algorithm \ref{alg:OptBias}, which is basically an implementation of the Eckart-Young Theorem.
	
	\begin{algorithm}
		\caption{$OptBias(C)$}
		\label{alg:OptBias}
		\begin{algorithmic}
			\STATE{{\bfseries Input:} Matrix $C\in\mathbb{R}^{(r+1)\times (r+1)}$}
			\STATE{{\bfseries Output:} Matrices $L',R'\in\mathbb{R}^{(r+1)\times r}$, so that $C'=L'R'^T$ minimizes $\lVert C-C'\rVert$.}
			\vspace{3pt}
			\STATE{\bfseries /* Reduce to diagonal matrix $D$*/}
			\STATE $(D,U,V) \gets \mathrm{SVD}(C)$ 
			\STATE $(d_1,\ldots,d_{r+1}) \gets$ diagonal entries of $D$
			\vspace{3pt}
			\STATE {\bfseries /* Initialise $L',R'$ to approximate $D$*/}
			\STATE{$L',R'\gets 0$}
			\FOR{$1\leq i \leq r$}
			\STATE $L'_{i,i}, R'_{i,i}\gets \sqrt{d_i}$
			\ENDFOR
			\STATE{\bfseries{/*Approximate $C=UDV^T$*/}}
			\STATE $L'\gets UL',\quad R'\gets  VR'$
		\end{algorithmic}
	\end{algorithm}	

\subsubsection{Experiments}
	The last experiment has essentially two goals. The first is to illustrate that biased approximations are not really desirable when doing gradient descent. This becomes clear in the difference in performance between the biased version of OK and the unbiased ones. The second goal is to show that, throughout training, and not just for specific points as shown in the cosine plots, the gradient can be well approximated by an $r$-Kronecker-Sum, for small values of $r$. In particular, this indicates that the noise in $r$-U-OK is small. Figure~\ref{fig:cp_OK_bias_exp} shows that $16$-B-OK performs almost as well as $16$-U-OK. The performance of $1$-B-OK is far worse than the corresponding unbiased OK. For the experiment, we use the same setup as described in Section 4.1.1 of the main paper. Apart from that, the rank 1 algorithms shown in the plot have been run with a batch size of 256. We repeat each experiment 5 times and plot the mean and standard deviation for each.


	\begin{figure}[ht]
		\vskip 0.2in
		\begin{center}
			\centerline{\includegraphics[width=\columnwidth]{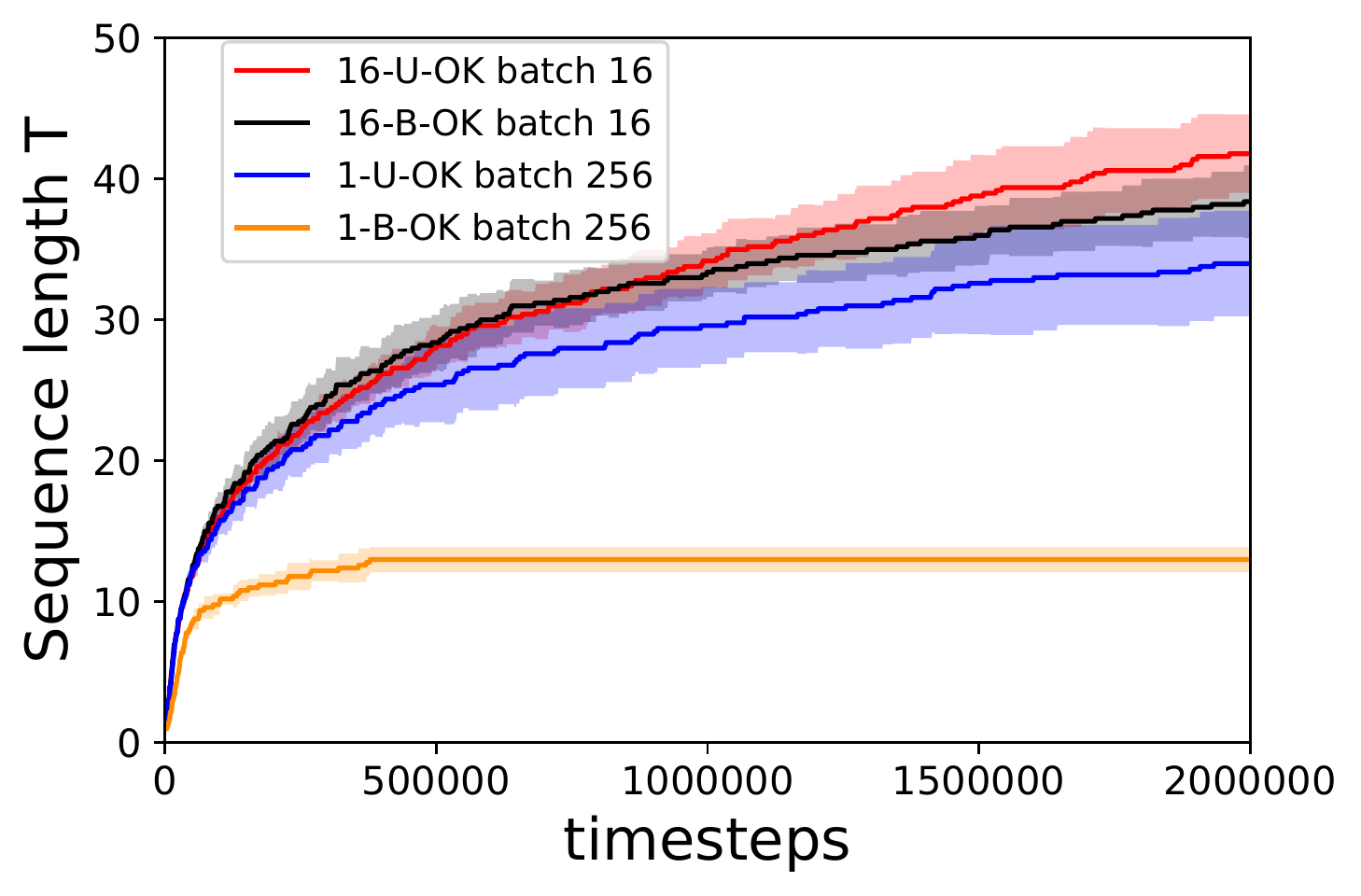}}
			\caption{Analysis of the Kronecker rank of the gradient. The biased rank 16 approximation of the gradient performs almost as well as the 16-OK. This implies the rank of the gradient throughout training can be well approximated by a sum of 16 Kronecker factors.
			}
			\label{fig:cp_OK_bias_exp}
		\end{center}
		\vskip -0.2in
	\end{figure}



\end{document}